\documentclass{article}

\usepackage[preprint]{neurips_2026}

\usepackage[utf8]{inputenc}
\usepackage[T1]{fontenc}
\usepackage{hyperref}
\usepackage{url}
\usepackage{booktabs}
\usepackage{amsmath,amssymb,amsthm}
\usepackage{mathtools}
\usepackage{graphicx}
\graphicspath{{figures/}}
\usepackage{bm}
\usepackage{xcolor}
\usepackage{multirow}
\usepackage{enumitem}
\usepackage{cleveref}
\usepackage{nicefrac}
\usepackage{microtype}

\newtheorem{theorem}{Theorem}
\newtheorem{proposition}[theorem]{Proposition}
\newtheorem{corollary}[theorem]{Corollary}
\newtheorem{lemma}[theorem]{Lemma}

\theoremstyle{remark}
\newtheorem{remark}[theorem]{Remark}
\newtheorem*{remark*}{Remark}

\newcommand{\R}{\mathbb{R}}
\newcommand{\E}{\mathbb{E}}
\newcommand{\norm}[1]{\left\| #1 \right\|}

\newcommand{\bx}{\bm{x}}
\newcommand{\bmu}{\bm{\mu}}
\newcommand{\bsig}{\bm{\sigma}}
\newcommand{\ba}{\bm{a}}

\newcommand{\FS}{\textsc{FluidSplat}}

\title{\textsc{FluidSplat}: Reconstructing Physical Fields from Sparse Sensors via Gaussian Primitives}

\author{%
  Huaxi Huang$^{1}$,
  Meng Li$^{1,2}$,
  Zhengqing Gao$^{1,3}$,
  Xi Zhou$^{1}$,
  Xiaoshui Huang$^{4}$,
  Xiao Sun$^{1}$\thanks{Corresponding Author}\\
  $^{1}$Shanghai Artificial Intelligence Laboratory\\
  $^{2}$The Hong Kong University of Science and Technology\\
  $^{3}$Mohamed bin Zayed University of Artificial Intelligence,
  $^{4}$Shanghai Jiaotong University
}

\begin{document}
\maketitle

\begin{abstract}
Reconstructing continuous flow fields from sparse surface-mounted sensors is central to aerodynamic design, flow control, and digital-twin instrumentation.  
Existing neural methods for this task typically encode sensor readings into implicit latent codes with little spatial interpretability and limited formal guidance on how representational capacity should scale with observation count. 
Inspired by 3D Gaussian Splatting, we introduce \FS{}, a sensor-conditioned model that predicts $K$
anisotropic Gaussian primitives forming a partition-of-unity scaffold, a spatially explicit and interpretable intermediate representation of
the flow.  
For an idealized Gaussian primitive estimator, we prove an $O(K^{-s/d})$ approximation rate for fields with Sobolev smoothness $s$; incorporating $N$ noisy observations yields a squared-risk
decomposition with bias $O(K^{-2s/d})$ and variance
$O(\sigma^{2}K/N)$.
Balancing the two yields $K^{*}\!\sim\!(N/\sigma^{2})^{d/(2s+d)}$: primitive
count cannot grow freely under sparse sensing, revealing a variance bottleneck that motivates complementing the scaffold with a state-conditioned residual decoder.
Across four benchmarks spanning 2D and 3D, \FS{} achieves 11--28\% error reduction over several strong baselines on cylinder flow, AirfRANS, FlowBench LDC-3D, and PhySense-Car 3D benchmarks.


\end{abstract}

\section{Introduction}
\label{sec:intro}

Many physical systems are instrumented by a small number of sensors attached to the solid surfaces: pressure taps on an airfoil, strain gauges on a turbine blade, thermocouples on a reactor wall~\cite{manohar2018,callaham2019,liu2025dspo}.  Recovering the full surrounding flow field from such sparse, boundary-mounted measurements is essential
for aerodynamic design, real-time control, and the construction of physical digital twins, yet it remains statistically and geometrically challenging: the observations are few, spatially confined to the
boundary, and provide no direct measurements of the interior.

Several neural architectures target sparse-sensor field
reconstruction directly: attention-based sensor
encoders~\cite{senseiver2023},
Voronoi-tessellation networks~\cite{fukami2021},
resolution-invariant spectral
models~\cite{zhao2024recfno}, Fourier-feature latent reconstruction
models~\cite{nguyen2026flrnet}, and shallow grid
decoders~\cite{erichson2020}.  In parallel, operator-learning
frameworks~\cite{lu2021,li2021,azizzadenesheli2024neural} and
geometry-aware or transformer-based PDE
solvers~\cite{li2023geofno,li2023gino,hao2023gnot,wu2024transolver}
have advanced surrogate modeling on irregular meshes, while
differentiable sensor-placement frameworks~\cite{liu2025dspo,ma2025physense}
co-optimize where to observe and how to reconstruct; however, they primarily target dense-grid, full-field, or placement-centric settings rather than reconstruction under a fixed sparse surface layout.  For sparse surface sensing specifically, existing methods typically do not jointly provide (i) a spatially explicit state for the reconstructed field and (ii) formal characterization of how model capacity should scale with the number of observations, leaving the choice of representational complexity to heuristic tuning strategy.

These observations motivate an alternative design principle.  Under
sparse surface sensing, it is natural to adopt an inductive bias toward
coherent spatial structures (wakes, vortices,
near-wall pressure regions) that the observations can plausibly
constrain, while delegating fine-scale detail to a flexible decoder.
This suggests a structured primitive scaffold that organizes the
recoverable spatial structure explicitly.  Anisotropic Gaussian
primitives are a natural basis for such a scaffold: spatially explicit,
local, and parameterized by a finite set of interpretable quantities.
In neural rendering, 3D Gaussian Splatting~\cite{kerbl2023} and its
variants~\cite{huang2024_2dgs,yu2024mip} have shown that such
primitives can efficiently represent radiance and geometry in visual
scenes.  We adopt the same representational bias for physical field
reconstruction, using the primitives not as a standalone field model
but as a \textbf{sensor-conditioned structured scaffold} paired with a
residual decoder for the remaining correction.

Concretely, we introduce \FS{}, a sensor-conditioned framework that
represents the spatial state through $K$ learnable anisotropic Gaussian
primitives.  A sensor encoder maps the sparse observations to a global
context vector; a linear head projects this context into the primitive
parameters (centers, anisotropic scales, weights, and amplitudes),
forming a partition-of-unity
field~\cite{babuska1997}.  A Fourier-feature residual
decoder~\cite{tancik2020}, conditioned on the global context and
scaffold state, then
supplies the correction that lies beyond the primitive
span.  Unlike a standard latent decoder, the primitive layer yields a
spatially explicit intermediate field that is sensor-conditioned; the residual decoder then refines the
reconstruction beyond this scaffold.

To ground this design theoretically, we analyze idealized Gaussian
primitive estimators and develop, to our knowledge, the first
approximation--estimation account of splatting-style
sensor-conditioned field reconstruction.  The analysis separates three
effects.  First, unnormalized Gaussian primitive sums achieve an
$L^2$ approximation rate $O(K^{-s/d})$ for Sobolev-regular fields.
Second, the Shepard-normalized scaffold is stable and spatially
explicit, but its oracle approximation saturates at first order because
it reproduces constants but not higher-order polynomials.  Third, when
$K$ primitive coefficients are estimated from $N$ noisy observations, an
idealized fixed-dictionary least-squares model gives a bias--variance
trade-off with squared bias $O(K^{-2s/d})$ and variance
$O(\sigma_{\mathrm{noise}}^2K/N)$.  Balancing these terms yields the
capacity scale
$K^{*}\!\sim\!(N/\sigma_{\mathrm{noise}}^{2})^{d/(2s+d)}$ up to
problem-dependent constants.  In the sparse-sensor regime, this scale
suggests that primitive capacity cannot be increased freely, motivating
a compact Gaussian scaffold complemented by a residual decoder.

Empirically, \FS{} is evaluated on four public benchmarks spanning 2D and 3D.  On a
standard cylinder-flow benchmark~\cite{senseiver2023} it achieves the
best mean error across all surface-sensor layouts; on the published
interior-8 protocol it reaches $0.0097$ relative $L^{2}$ (reference
value: $0.039$).  On
AirfRANS~\cite{airfrans2022} with 8 surface-pressure sensors, it
reduces error by 11--23\% over the strongest baseline across three
generalization splits.
On FlowBench LDC-3D~\cite{flowbench2024}, \FS{} extends these gains to volumetric reconstruction with 28\% improvement over the strongest baseline.  As an additional 3D test, we further evaluate \FS{} on the released PhySense-Car sensor-only protocol~\cite{ma2025physense} for surface-pressure reconstruction, where it improves over the released baselines under the same evaluation pipeline.  In summary, our contributions are threefold:
\begin{enumerate}
\item We propose \FS{}
(Section~\ref{sec:method}), a sensor-conditioned architecture that
combines an explicit Gaussian primitive scaffold with a
state-conditioned residual decoder for sparse-sensor flow
reconstruction.
\item We develop an approximation--estimation analysis that explains the expressivity and sparse-sensing capacity limits of Gaussian primitive scaffolds, providing theoretical justification for the proposed
architecture (Section~\ref{sec:theory}; proofs in Appendices~\ref{app:approx}--\ref{app:estim}).
\item  On four public benchmarks (2D and 3D) under the reported evaluation protocols, \FS{} achieves the best mean performance among all evaluated methods, with 11--23\% error reduction on AirfRANS, 28\% on FlowBench LDC-3D, and additional gains on the released PhySense-Car sensor-only protocol (Section~\ref{sec:experiments}).
\end{enumerate}

\section{Related Work}
\label{sec:related}

\paragraph{Sparse-sensor flow reconstruction.}
Classical approaches combine data-driven sensor
placement~\cite{manohar2018} with POD-based or sparse
reconstruction~\cite{callaham2019}.  Neural methods include shallow
decoders from limited sensors~\cite{erichson2020,fukami2020},
Voronoi-tessellation CNNs~\cite{fukami2021}, attention-based
architectures that treat sensors as a sparse token
set~\cite{senseiver2023}, and Fourier-feature or spectral models for
resolution-invariant reconstruction~\cite{zhao2024recfno,nguyen2026flrnet}.
More broadly, operator-learning
frameworks~\cite{lu2021,li2021,azizzadenesheli2024neural} and
geometry-aware PDE
solvers~\cite{li2023geofno,li2023gino,hao2023gnot,wu2024transolver}
advance surrogate modeling on general domains but are primarily
developed for forward prediction with full-field supervision, although
adaptable to reconstruction tasks;
differentiable placement-optimization
frameworks~\cite{liu2025dspo,ma2025physense} co-learn sensor locations
and reconstruction and are complementary to our fixed surface-mounted
sensor setting: they optimize \emph{where} to observe, whereas we
study representational capacity under a given layout.  Among existing sparse-sensor methods, most encode
observations into an implicit latent and decode on a grid or at query
points; they typically do not jointly provide (i)~a spatially explicit
intermediate state and (ii)~formal characterization of how
representational capacity should scale with the observation count.

\paragraph{Gaussian primitives and neural fields.}
3DGS~\cite{kerbl2023}, descended from EWA
splatting~\cite{zwicker2002}, represents scenes via anisotropic
Gaussian primitives and has become a standard in neural
rendering~\cite{mildenhall2020,huang2024_2dgs,yu2024mip,chen2024survey}.
These primitives are spatially explicit, differentiable, and support
efficient rendering and evaluation, but have been developed primarily
for visual radiance and geometry.  Scattered-data approximation with
radial basis functions~\cite{wendland2005,schaback2006,nww2005} and
partition-of-unity methods~\cite{babuska1997} provide the classical
foundation for Gaussian-kernel approximants.  We are not aware of
prior work that employs splatting-style primitives as a
sensor-conditioned spatial scaffold for physical-field reconstruction
or analyses their capacity in this sparse-observation regime.

\paragraph{Approximation-theoretic context.}
Neural-network approximation theory~\cite{mhaskar1996,yarotsky2017}
and minimax estimation rates~\cite{stone1982,ibragimov2013}
supply the statistical tools for analyzing Gaussian-primitive
capacity.  Existing sparse-sensor methods do not, to our knowledge,
connect the capacity of an explicit spatial representation to
observation count via a bias-variance decomposition with a $K/N$
variance term; we provide this connection and show that it yields a
variance bottleneck motivating the residual decoder.

\section{\FS: Method}
\label{sec:method}

\subsection{Problem Setting}

Let $\Omega\subset\R^{d}$ be a bounded Lipschitz domain and
$\bm{f}:\Omega\to\R^{C}$ the target vector field
(e.g., velocity components plus pressure, $C=d+1$).
We are given $N$ noisy pointwise observations
\begin{equation}
\label{eq:obs}
y_{i}\;=\;\bm{f}(\bx_{i}^{\mathrm{obs}}) + \varepsilon_{i},\qquad
i=1,\dots,N,
\end{equation}
with $\bx_{i}^{\mathrm{obs}}$ typically located on the body surface
(e.g., surface-mounted pressure sensors) and noise satisfies mean zero $\mathbb E[\varepsilon_i]=0$ and bounded variance $\operatorname{Cov}(\varepsilon)\preceq\sigma^2_{\mathrm{noise}}  I_N$.
The task is to return a continuous reconstruction
$\hat{\bm{f}}(\bx)$ queryable at arbitrary $\bx\in\Omega$.

\subsection{Overview}

\FS{} maps a sparse sensor set to a continuous field through three
stages: (i) a
\emph{permutation-invariant sensor encoder} produces a global context
vector $\bm{z}$ from the unordered observation set; (ii) a linear head
predicts $K$ anisotropic Gaussian primitive parameters from $\bm{z}$,
forming a normalized partition-of-unity \emph{scaffold} field
$\bm{f}_{\mathrm{prim}}$; (iii) a \emph{state-conditioned residual
decoder}, whose inputs include the scaffold evaluation and
query-to-primitive cross-attention, supplies the correction that lies
beyond the scaffold span.  The scaffold is the
component analyzed theoretically in Section~\ref{sec:theory}; the
residual decoder is motivated by the variance bottleneck identified
therein.

\begin{figure}[t]
    \centering
    \includegraphics[width=0.9\linewidth]{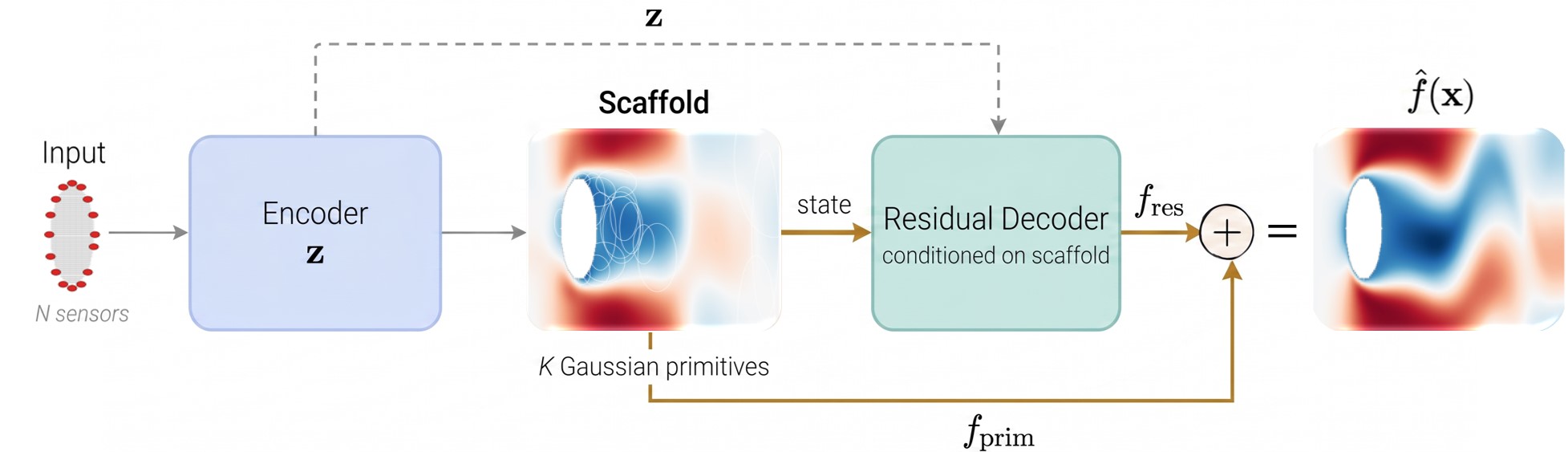}
    \caption{Overview of \FS{}. A permutation-invariant encoder maps $N$
    surface sensors to a global context $\bm{z}$, which parameterizes $K$
    anisotropic Gaussian primitives forming the scaffold field
    $\bm{f}_{\mathrm{prim}}$.  A residual decoder, conditioned on the
    scaffold state, produces $\bm{f}_{\mathrm{res}}$; the final output is
    $\hat{\bm{f}}(\bx) = \bm{f}_{\mathrm{prim}}(\bx) +
    \bm{f}_{\mathrm{res}}(\bx)$.  Insets show real model outputs on
    cylinder flow.}
    \label{fig:framework}
\end{figure}

\subsection{Sensor-Conditioned Gaussian Scaffold}
\label{sec:scaffold}

\paragraph{Primitive definition.}
Each of $K$ primitives is parameterized by center
$\bmu_{k}\in\Omega$, positive-definite precision structure
$\Sigma_{k}^{-1}$, weight $w_{k}\in(0,1)$, and
amplitude $\ba_{k}\in\R^{C}$.  The unnormalized Gaussian basis is
\begin{equation}
\label{eq:basis}
\varphi_{k}(\bx)
\;=\;
\exp\!\Bigl(
-\tfrac{1}{2}\,
(\bx - \bmu_{k})^{\!\top}\Sigma_{k}^{-1}(\bx - \bmu_{k})
\Bigr).
\end{equation}
In our 2D experiments, $\Sigma_{k}$ is parameterized by axis scales
$\bsig_{k}\in(0,\infty)^{2}$ and an optional rotation angle
$\theta_{k}$, giving
$\Sigma_{k}^{-1}=R_{\theta_k}^{\!\top}\,
\mathrm{diag}(\bsig_{k}^{-2})\,R_{\theta_k}$.

\paragraph{Partition-of-unity field.}
The scaffold is a normalized mixture
\begin{equation}
\label{eq:model}
\hat{\bm{f}}(\bx)
\;=\;
\underbrace{\sum_{k=1}^{K}\psi_{k}(\bx)\,\ba_{k}}
_{\bm{f}_{\mathrm{prim}}(\bx)}
\;+\;
\underbrace{\bm{f}_{\mathrm{res}}(\bx;\,\bm{z},\mathcal{S})}
_{\text{residual decoder}},
\qquad
\psi_{k}(\bx) =
\frac{w_{k}\varphi_{k}(\bx)}{\sum_{j}w_{j}\varphi_{j}(\bx)},
\end{equation}
with $\psi_{k}\ge 0$, $\sum_{k}\psi_{k}=1$, and
$\mathcal{S}=\{(\bmu_k,\Sigma_k,w_k,\ba_k)\}_{k=1}^K$ denoting the
full Gaussian state.
The scaffold provides a structured, spatially explicit, sensor-conditioned
intermediate field: primitives provide localized degrees of freedom
that can allocate support to coherent spatial structures, with
parameters directly predicted from the observations.  The scaffold is
not intended to fully reconstruct the field by itself; rather, it
exposes a structured spatial state that conditions the residual decoder.

\paragraph{Sensor encoder and primitive head.}
The encoder processes the unordered sensor set
$\{(\bx_{i}^{\mathrm{obs}},y_{i})\}_{i=1}^{N}$ via a shared
per-token MLP followed by mean and max pooling, yielding a global
context $\bm{z}\in\R^{D}$.  The aggregation is
permutation-invariant and does not require changing network dimensions
when the sensor count varies.
A linear projection maps $\bm{z}$ to all $K$ primitive parameter
vectors; for the 2D parameterization used in our experiments:
\begin{equation}
\{(\bmu_{k},\bsig_{k},\theta_{k},w_{k},\ba_{k})\}_{k=1}^{K}
\;=\;
\mathrm{Head}(\bm{z}),
\end{equation}
with centers mapped via sigmoid to the normalized computational
domain, scales via bounded sigmoid to $[\sigma_{\min},\sigma_{\max}]$,
and weights via sigmoid.  For complex geometries,
optional geometry tokens and case-metadata conditioning are fused into
$\bm{z}$ before the primitive head.

\subsection{State-Conditioned Residual Decoder}
\label{sec:resid}

The residual decoder refines the scaffold by receiving both global and
local state information at each query point $\bx$.  Its input
concatenates Fourier-encoded coordinates
$\gamma(\bx)$~\cite{tancik2020}, the global context $\bm{z}$, and the
scaffold evaluation $\bm{f}_{\mathrm{prim}}(\bx)$ together with the
basis mass $m(\bx)=\sum_{k}w_{k}\varphi_{k}(\bx)$, which signals
local scaffold support.  To allow the decoder to attend directly to the
primitive state, query tokens (formed from $\gamma(\bx)$ and $\bm{z}$)
attend to primitive-parameter tokens via multi-head cross-attention with
a residual connection and feed-forward layer.  An optional learned global
token projected from $\bm{z}$ is appended when used.

The concatenated features are passed through an MLP to produce the
residual $\bm{f}_{\mathrm{res}}(\bx;\bm{z},\mathcal{S})\in\R^{C}$.  Because the decoder
observes the scaffold field directly, it can condition its correction
on the primitive state---coupling the two components so that the
scaffold provides structured degrees of freedom while the decoder
supplies the remaining detail beyond the scaffold span.
The core architecture is shared across all experiments;
dataset-specific information enters only through lightweight
conditioning channels: the covariance parameterization of the Gaussian
state, and optional per-query geometry or metadata features when such
quantities are part of the benchmark protocol (Appendix~\ref{app:exp_setup}).

\subsection{Training and Inference}

\paragraph{Training.}
We train end-to-end with full-field supervision at randomly sampled
query points plus a soft sensor-consistency term:
\begin{equation}
\label{eq:loss}
\mathcal{L}
=
\E_{\bx\sim\Omega}\norm{\hat{\bm{f}}(\bx)-\bm{f}(\bx)}^{2}
+
\lambda_{\mathrm{obs}}\,
\tfrac{1}{N}\textstyle\sum_{i}\norm{\hat{\bm{f}}(\bx_{i}^{\mathrm{obs}})-y_{i}}^{2}.
\end{equation}

\paragraph{Inference.}
At test time only the $N$ sensor readings are provided; the model
produces a continuous field queryable at any resolution or mesh.  All
hyperparameters and dataset-specific conditioning details are given in
Appendix~\ref{app:exp_setup} and \ref{app:additional_tables}.

\section{Theoretical Analysis}
\label{sec:theory}

This section states the main results.  Full proofs, the dimension-specific analysis, and the connection to Navier--Stokes regularity are in Appendices~\ref{app:approx}--\ref{app:ns}.  

Let $\Omega \subset \mathbb{R}^d$ be the bounded spatial domain, $d$ the spatial dimension, and $s$ the Sobolev regularity of the target field. We write $a\lesssim b$ when $a \leq Cb$, where $C$ is independent of the primitive count $K$, sensor count $N$, and noise level $\sigma_\mathrm{noise}$. Geometric assumptions on primitive centers $\mu_k$, including fill distance $h$, separation distance $q$, and quasi-uniformity, are stated only when needed and formalized in Appendix~\ref{app:notation}.

\subsection{Representational Capacity}

\begin{theorem}[Gaussian primitive approximation]
\label{thm:rate}
Let $\Omega\subset\mathbb{R}^{d}$ be a bounded Lipschitz domain and
let $\varepsilon>0$.  For every $f\in H^{s+\varepsilon}(\Omega)$ with
$s>0$ and every $K\ge 1$, there exists a $K$-term unnormalized
Gaussian approximant $g_{K}=\sum_{k}c_{k}\varphi_{k}$ such that
\begin{equation}
\label{eq:rate}
\|f-g_{K}\|_{L^{2}(\Omega)}
\;\le\;
C_{\varepsilon,s,d,\Omega}\,K^{-s/d}\,
\|f\|_{H^{s+\varepsilon}(\Omega)}.
\end{equation}
The vector-valued case
$\bm{f}\in H^{s+\varepsilon}(\Omega;\mathbb{R}^{C})$ follows
componentwise.
\end{theorem}

The arbitrarily small $\varepsilon>0$ is the technical cost of passing from Sobolev spaces to the nonlinear Gaussian approximation class. The proof specializes nonlinear Gaussian-network approximation rates of Hangelbroek and Ron~\cite{hangelbroek2010nonlinear} to bounded Lipschitz domains via a Sobolev embedding argument; see Appendix~\ref{app:approx}.

\begin{proposition}[Partition-of-unity saturation]
\label{prop:norm_rate}
Assume that the Gaussian centers are quasi-uniform, the weight satisfy $0<w_-\le w_k\le w_+<\infty$, and the Gaussian covariances are comparable to the squared fill distance $ \Sigma_k \asymp h_K^2 I$. For $f\in H^s(\Omega), s>1+d/2$, the oracle scaffold $f_{\mathrm{prim}}$ in~\eqref{eq:model}, with $a_k=f(\bmu_k)$, satisfies
\[
\|f-f_{\mathrm{prim}}\|_{L^2(\Omega)}
\lesssim
K^{-1/d}\|f\|_{H^s(\Omega)}.
\]
\end{proposition}
The first-order saturation reflects zeroth-order polynomial reproduction of the partition of unity, necessitating a residual component for higher-order correction.

\begin{remark}[Scope of novelty]
\label{rem:first_rate}
Theorem~\ref{thm:rate} specializes existing nonlinear Gaussian
approximation theory to the primitive scaffold used in \FS{}, while
Proposition~\ref{prop:norm_rate} isolates the first-order saturation of
the Shepard-normalized field.  Together with the estimation analysis
below, these results provide, to our knowledge, the first
approximation--estimation account of splatting-style
sensor-conditioned field reconstruction, linking Gaussian primitive
count to observation noise through an explicit capacity scaling.
\end{remark}

\subsection{Variance Bottleneck}

\paragraph{Ordinary least squares.} Let $V_{K}=\mathrm{span}\{\varphi_{k}\}_{k=1}^{K}$ be a $K$-dimensional Gaussian dictionary with fixed centers $\mu_k$ and precision matrices $\Sigma_k^{-1}$. Only the linear coefficients are estimated from data. Assume that the observation matrix $\bm{A}\in\R^{N\times K}$, $A_{ij}=\varphi_{j}(\bx_{i}^{\mathrm{obs}})$ has full column rank. Let $\bm{G}\in\R^{K\times K}$ with $G_{jl}=\int_{\Omega}\varphi_{j}\varphi_{l}\mathrm{d}\bx$ denote the Gram matrix and $r_X:=(r_1,\ldots,r_N)^\top$ with $r_i=f(x_i)-f_K^*(x_i)$ denote the approximation residual. Define the least-squares estimator
\[
\hat f_K(x)=\sum_{k=1}^K\hat c_k\varphi_k(x) \text{, where }  \hat c=\arg\min_{c\in\mathbb R^K}
\frac1N\sum_{i=1}^N
\left(y_i-\sum_{k=1}^K c_k\varphi_k(x_i)\right)^2.
\]

\begin{theorem}[Bias--variance bound]
\label{thm:bv}

Assume (i) the fixed dictionary \(V_K\) admits the approximation bound; (ii) the samples are spectrally stable $A^\top A\asymp NG$; (iii) the approximation residual at sensor locations is controlled $\|r_X\|_2/\sqrt{N}\lesssim \|r\|_{L^2(\Omega)}$. Then
\begin{equation}
\label{eq:bv}
\E\norm{\bm{f}-\hat{\bm{f}}_{K}}_{L^{2}}^{2}
\;\le\;
\underbrace{C_{\mathrm{bias}}\,K^{-2s/d}\norm{\bm{f}}_{H^{s+\varepsilon}}^{2}}
_{\mathrm{bias}^{2}}
\;+\;
\underbrace{C_{\mathrm{var}}\,\sigma^{2}_{\mathrm{noise}} \,K/N}
_{\mathrm{variance}}.
\end{equation}
where $C_{\mathrm{bias}}$ and $C_{\mathrm{var}}$ are constants independent of $K, N, \sigma_{\mathrm{noise}}$.
\end{theorem}

\begin{corollary}[Capacity scale]
\label{thm:opt_K}
Minimizing the upper bound in~\eqref{eq:bv} gives
\begin{equation}
\label{eq:opt_K}
K^{*}\sim(N/\sigma^{2}_{\mathrm{noise}} )^{d/(2s+d)}\,
\norm{\bm{f}}_{H^{s+\varepsilon}}^{2d/(2s+d)},
\end{equation}
At this scale,
\[
\mathbb E\|f-\hat f_{K^*}\|_{L^2(\Omega)}^2
\lesssim
\left(\frac{\sigma_{\mathrm{noise}}^2}{N}\right)^{2s/(2s+d)}
\|f\|_{H^{s+\varepsilon}}^{2d/(2s+d)} .
\]
\end{corollary}

\begin{remark}[Why residual decoders help]
\label{rem:residual-capacity}
Increasing $K$ reduces approximation bias by enlarging the Gaussian dictionary, but it also increases the number of coefficients estimated from noisy observations.  In the fixed-dictionary least-squares model, this produces a variance term proportional to $K/N$.  The resulting capacity scale suggests that, under sparse sensing, the primitive field should remain compact.  This motivates using the Gaussian primitives as a spatial scaffold, with the residual decoder modeling corrections not captured by this low-dimensional component.
\end{remark}

\section{Experiments}
\label{sec:experiments}

\paragraph{Datasets.}
\emph{Senseiver cylinder}: 2D cylinder vorticity
benchmark, evaluated under (a) the official interior-8
protocol~\cite{senseiver2023} and (b) our surface-4/8/16 layouts on
the cylinder wall.
\emph{AirfRANS}~\cite{airfrans2022}: 2D RANS airfoil simulations with
surface-pressure sensors uniformly spaced in arclength; three splits
--- Full, AoA-shift, Re-shift.
\emph{FlowBench LDC-3D}~\cite{flowbench2024}: 3D lid-driven cavity with
random internal obstacles on a $64^3$ grid; $\text{Re}\in[10,997]$;
four output channels ($u,v,w,p$); 700/100/200 train/val/test split.
AirfRANS and cylinder-surface results
use fixed layouts and five training seeds; LDC-3D uses five seeds; the interior-8 Senseiver
baseline is quoted from the published paper.  Relative $L^{2}$ error
is the primary metric.  See Appendix~\ref{app:exp_setup} for complete details.

\paragraph{Baselines.}
\emph{Senseiver official} for the cylinder protocol (official code
adapted to the surface layouts with a training-length sweep, detailed in
Appendix~\ref{app:baselines});
\emph{Senseiver official-adapter} for AirfRANS; \emph{DeepONet}; and two
operator-style re-implementations (\emph{RecFNO} and \emph{FLRONet},
labeled ``reimpl.'' throughout).  POD-Ridge is included for
AirfRANS as a classical baseline.

\paragraph{Reporting convention.}
Unless explicitly stated otherwise, main-table entries report the mean
and standard deviation over training seeds.  When a best-of-five value
is shown in parentheses, it is provided only to compare with
single-run reporting conventions; rankings are based on seed means.

\paragraph{Evaluation order.}
We present results from simple to realistic.  The Senseiver cylinder
benchmark is a controlled 2D proof-of-concept where the flow geometry,
field variable, and train/test split are simple enough to isolate the
sensing architecture.  AirfRANS is the more realistic aerodynamic
benchmark: irregular meshes, airfoil geometry, multiple physical
channels, and out-of-distribution AoA/Reynolds splits make it the
stronger test of robustness.  FlowBench LDC-3D extends the evaluation
to full 3D volumetric reconstruction with complex obstacle geometries.

\begin{table}[t]
\centering
\caption{Controlled cylinder benchmark.  Relative $L^{2}$ is evaluated
on 4{,}950 test frames.  Interior-8 uses the official Senseiver sensor
layout and protocol; the Senseiver entry is the published value, while
all other Interior-8 entries are single-run local evaluations under the
same protocol.  Surface columns report mean $\pm$ seed std over seeds
$\{123,\dots,127\}$.  Ranking is by seed mean for surface results;
lower is better.}
\label{tab:cylinder}
\scriptsize
\resizebox{\linewidth}{!}{%
\begin{tabular}{lcccc}
\toprule
Method & Interior-8 (protocol) & Surface-4 & Surface-8 & Surface-16 \\
\midrule
\FS{} (ours)            & \textbf{0.0097} & \textbf{0.0642 $\pm$ 0.0041} & \textbf{0.0358 $\pm$ 0.0073} & \textbf{0.0160 $\pm$ 0.0008} \\
Senseiver official      & 0.039  & 0.0759 $\pm$ 0.0121 & 0.0580 $\pm$ 0.0071 & 0.0556 $\pm$ 0.0109 \\
DeepONet                & 0.3847 & 0.0925 $\pm$ 0.0034 & 0.1301 $\pm$ 0.0137 & 0.0946 $\pm$ 0.0023 \\
FLRONet (reimpl.)       & 0.3715 & 0.0693 $\pm$ 0.0066 & 0.0678 $\pm$ 0.0099 & 0.0460 $\pm$ 0.0037 \\
RecFNO (reimpl.)        & 0.0875 & 0.0745 $\pm$ 0.0162 & 0.1024 $\pm$ 0.0112 & 0.0705 $\pm$ 0.0056 \\
\bottomrule
\end{tabular}%
}
\end{table}

\begin{table}[t]
\centering
\caption{AirfRANS main results, 8-sensor surface-pressure layout,
five training seeds.  Lower is better; best in bold.}
\label{tab:airfrans}
\small
\begin{tabular}{lccc}
\toprule
Method & Full-8 & AoA-8 & Reynolds-8 \\
\midrule
\FS{} (ours)         & \textbf{0.2518 $\pm$ 0.0389}
                     & \textbf{0.5280 $\pm$ 0.1010}
                     & \textbf{0.3129 $\pm$ 0.0229} \\
DeepONet             & 0.3289 $\pm$ 0.0330 & 0.5937 $\pm$ 0.0991 & 0.3859 $\pm$ 0.0183 \\
Senseiver official-adapter & 0.3891 $\pm$ 0.0311 & 0.7546 $\pm$ 0.0904 & 0.4647 $\pm$ 0.0674 \\
RecFNO (reimpl.)     & 0.3939 $\pm$ 0.0503 & 0.8078 $\pm$ 0.1297 & 0.4400 $\pm$ 0.0404 \\
FLRONet (reimpl.)    & 0.4339 $\pm$ 0.0323 & 0.8013 $\pm$ 0.0706 & 0.4710 $\pm$ 0.0461 \\
POD-Ridge            & 0.5849              & 1.2982              & 0.6024              \\
\bottomrule
\end{tabular}
\end{table}

\begin{table}[!htpb]
\centering
\caption{FlowBench LDC-3D results (5 seeds).  Relative $L^2$ over all
four channels ($u,v,w,p$); lower is better.  Best in bold.}
\label{tab:ldc3d}
\scriptsize
\resizebox{\linewidth}{!}{%
\begin{tabular}{lcccc}
\toprule
Method & $N\!=\!5$ & $N\!=\!10$ & $N\!=\!20$ & $N\!=\!40$ \\
\midrule
\FS{} (ours) & \textbf{0.0307 $\pm$ 0.0009} & \textbf{0.0288 $\pm$ 0.0009} & \textbf{0.0286 $\pm$ 0.0010} & \textbf{0.0284 $\pm$ 0.0006} \\
FLRONet (reimpl.) & 0.0435 $\pm$ 0.0010 & 0.0397 $\pm$ 0.0004 & 0.0399 $\pm$ 0.0008 & 0.0402 $\pm$ 0.0010 \\
DeepONet & 0.0518 $\pm$ 0.0018 & 0.0606 $\pm$ 0.0212 & 0.0485 $\pm$ 0.0046 & 0.0459 $\pm$ 0.0043 \\
Senseiver (reimpl.) & 0.0573 $\pm$ 0.0025 & 0.0854 $\pm$ 0.0347 & 0.0542 $\pm$ 0.0055 & 0.0560 $\pm$ 0.0039 \\
RecFNO (reimpl.) & 0.0850 $\pm$ 0.0325 & 0.0678 $\pm$ 0.0007 & 0.0679 $\pm$ 0.0006 & 0.0753 $\pm$ 0.0158 \\
POD-Ridge & 1.0998 & 1.0225 & 1.0415 & 1.2003 \\
\bottomrule
\end{tabular}%
}
\end{table}

\subsection{Senseiver Cylinder: Controlled Concept Validation}

Table~\ref{tab:cylinder} compares methods on the Senseiver cylinder
benchmark under the official interior-8 protocol and our new
surface-sensor protocol.  For Interior-8, we report \FS{} under the
same layout and evaluation metric and quote the published Senseiver
value $0.039$~\cite{senseiver2023} as the official reference.  Under the surface
protocol, \FS{} obtains the best five-seed mean for every sensor
count.  Exact per-seed surface results are provided in
Appendix~\ref{app:additional_tables}; the main table reports
seed-level statistics.

\subsection{AirfRANS: Realistic Sparse Surface-Pressure Reconstruction}

Table~\ref{tab:airfrans} reports the primary metric (clean
reconstruction, mean $\pm$ std over 5 seeds) on the 8-sensor
surface-pressure layout.  \FS{} is the best method in every setting
with 11.1--23.4\% relative improvement over the strongest non-ours
baseline.  Relative to the Senseiver official-adapter baseline, the
improvement is 30.0--35.3\% across the three splits.

\subsection{FlowBench LDC-3D: 3D Volumetric Reconstruction}

\begin{figure}[!htpb]
\centering
\small
\textbf{(a)} Cylinder Surface-8, vorticity\\[2pt]
\includegraphics[width=\linewidth]{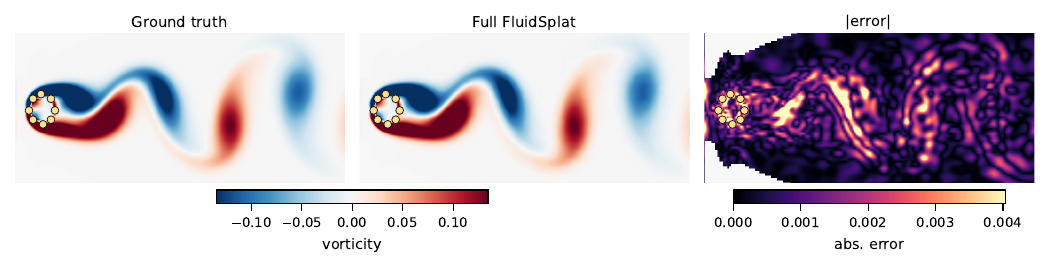}\\[6pt]
\textbf{(b)} AirfRANS Full-8, centered pressure\\[2pt]
\includegraphics[width=\linewidth]{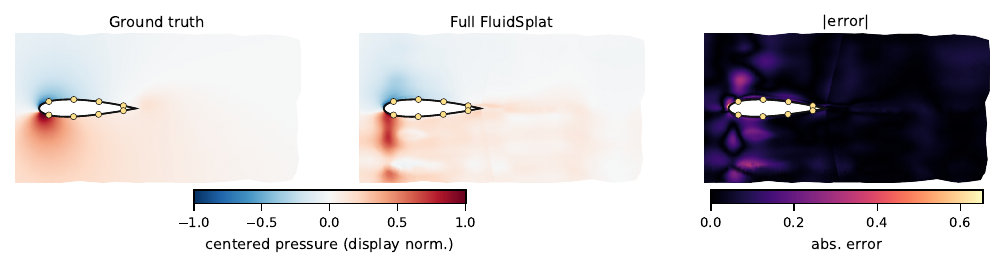}\\[6pt]
\textbf{(c)} FlowBench LDC-3D, velocity streamlines\\[2pt]
\includegraphics[width=0.82\linewidth]{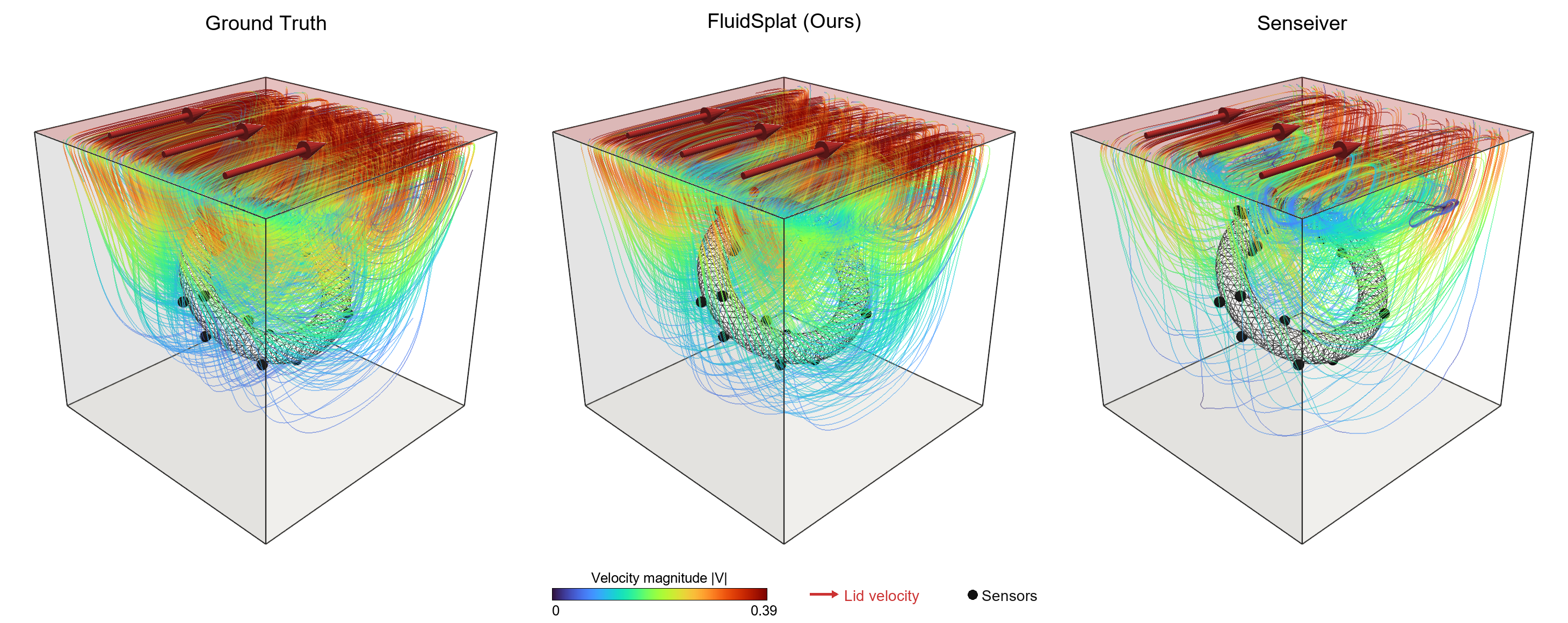}
\caption{Qualitative reconstructions on the cylinder, AirfRANS, and
FlowBench LDC-3D benchmarks.  Each panel shows ground truth, \FS{}
prediction, and absolute error or baseline comparison.  Gold dots mark
sensors in (a)--(b); \FS{} recovers fine-scale structures while
baselines are overly smooth.  PhySense-Car is shown separately in
Figure~\ref{fig:physense_car}.}
\label{fig:qualitative_all}
\end{figure}

We extend evaluation to three dimensions using the FlowBench
LDC-3D benchmark~\cite{flowbench2024}: 1{,}000 lid-driven cavity
simulations on a $64^3$ grid with randomly placed internal obstacles
and Reynolds numbers spanning $\text{Re}\in[10, 997]$.  Each case has
four output channels (velocity $u,v,w$ and pressure $p$) and the
obstacle geometry is encoded as a signed distance field (SDF).  We use
a 700/100/200 train/val/test split.  Sensors are placed uniformly at
random among the fluid cells.

Table~\ref{tab:ldc3d} reports the relative $L^2$ error across all
four channels, averaged over five training seeds, as the sensor count
$N$ varies from 5 to 40.  \FS{} achieves the lowest error at every
sensor count, with 28\% relative improvement over FLRONet (the
strongest baseline) at $N\!=\!20$.  All neural methods substantially
outperform POD-Ridge, whose relative error exceeds 1.0 due to the
complex 3D obstacle geometries.
Figure~\ref{fig:qualitative_all} (bottom) provides a qualitative comparison
via 3D streamlines for a representative test case; \FS{} preserves
the complex recirculation patterns visible in the ground truth, while
Senseiver produces an overly smooth velocity field.

\subsection{PhySense-Car: 3D Surface-Pressure Reconstruction}
\label{sec:physense_car}

To complement the volumetric LDC-3D benchmark, we evaluate \FS{} on
PhySense-Car~\cite{ma2025physense}, a 3D car-surface pressure
reconstruction benchmark built on ShapeNet cars and OpenFOAM
simulations with $\sim\!95$k surface points per case.  This setting is
complementary to LDC-3D in two ways: the field lives on an
unstructured 3D surface rather than a volumetric grid, and the
pressure dynamic range is much wider (mean $\approx-75$, std
$\approx149$).  We follow the official sensor-only protocol of
PhySense, using the official random-sensor seed and the official
normalized relative-$L^2$ metric, with cases 1--75 for training and
76--100 for testing.

Table~\ref{tab:physense_car} reports the normalized relative $L^2$
error at sensor counts $N\in\{15,30,50,100,200\}$, where the four
baseline numbers are taken directly from the PhySense
paper~\cite{ma2025physense} Table~3.  \FS{} improves over the
strongest released baseline (PhySense-opt) at every sensor count,
using a single checkpoint trained for 301 epochs and no test-time
optimization.
Figure~\ref{fig:physense_car} compares ground truth, \FS{}, and the
released PhySense base checkpoint on two held-out cases at $N=50$:
PhySense leaves spatially structured errors over the wheels, roof,
and rear body, while \FS{} concentrates errors in smaller
high-gradient regions.  Full protocol details (normalization, seed,
sensor masking) are in Appendix~\ref{app:physense_car}.

\begin{table}[!hbpt]
\centering
\caption{PhySense-Car surface-pressure reconstruction (normalized
relative $L^2$, lower is better).  Baseline numbers are quoted from
PhySense~\cite{ma2025physense} Table~3; \FS{} uses the same
sensor-only protocol with the official random-sensor seed.}
\label{tab:physense_car}
\small
\begin{tabular}{lccccc}
\toprule
Method & $N\!=\!15$ & $N\!=\!30$ & $N\!=\!50$ & $N\!=\!100$ & $N\!=\!200$ \\
\midrule
Senseiver           & 0.1053 & 0.1022 & 0.1018 & 0.1012 & 0.1009 \\
DiffusionPDE        & 0.2095 & 0.1055 & 0.0987 & 0.0966 & 0.0967 \\
PhySense            & 0.0465 & 0.0416 & 0.0395 & 0.0382 & 0.0375 \\
PhySense-opt        & 0.0386 & 0.0372 & 0.0370 & 0.0370 & 0.0369 \\
\midrule
\FS{} (ours)        & \textbf{0.0097} & \textbf{0.0086} & \textbf{0.0084}
                    & \textbf{0.0079} & \textbf{0.0073} \\
\bottomrule
\end{tabular}
\end{table}

\begin{figure}[!hbpt]
\centering
\includegraphics[width=\linewidth]{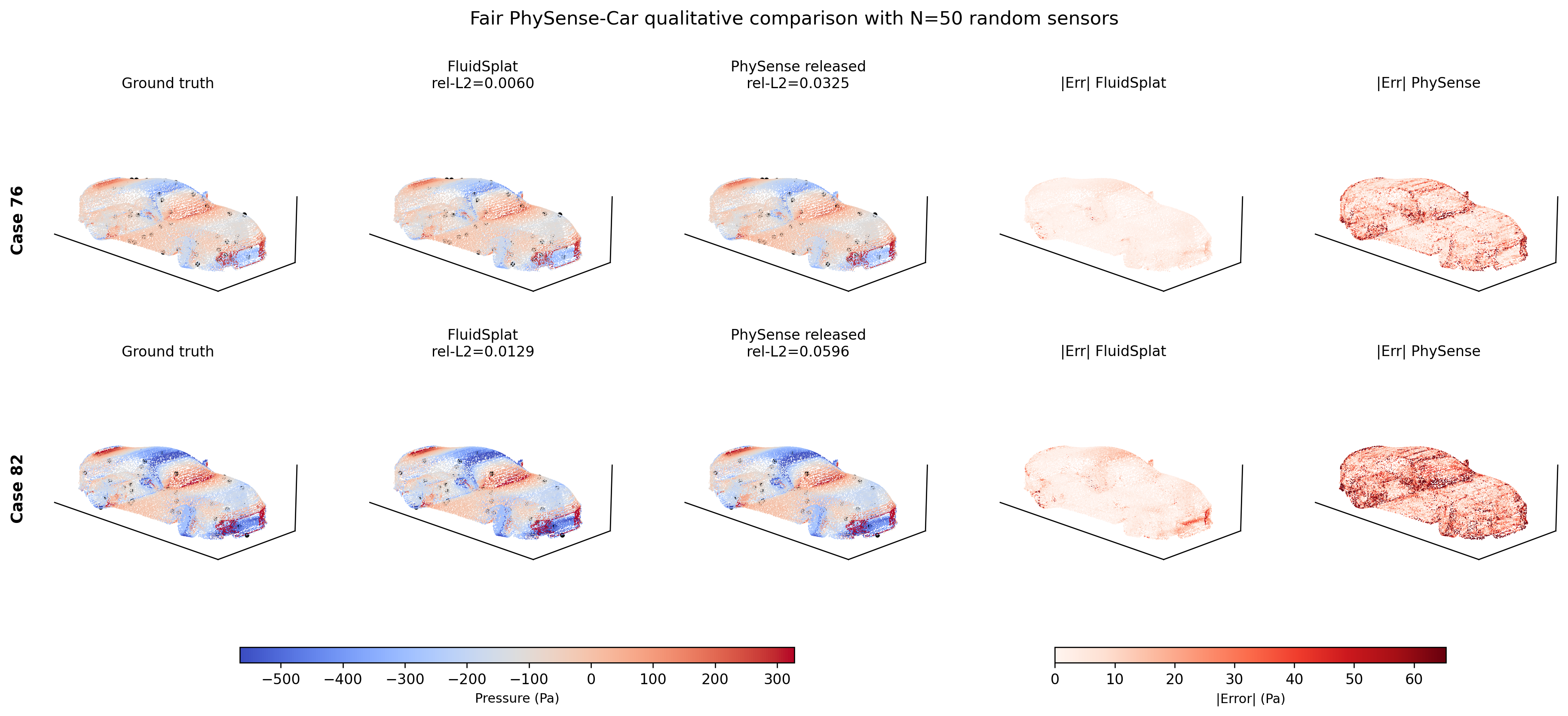}
\caption{PhySense-Car qualitative comparison under the released
random-sensor protocol at $N=50$.  Columns show ground truth, \FS{},
the released PhySense base-model prediction, absolute error of
\FS{}, and absolute error of PhySense.  Black dots mark identical
sampled pressure sensors in all prediction panels.}
\label{fig:physense_car}
\end{figure}

\subsection{Qualitative Reconstruction}
\label{sec:qualitative}
Figure~\ref{fig:qualitative_all} groups the qualitative reconstructions
on the cylinder, AirfRANS, and LDC-3D benchmarks: the cylinder vortex
street, AirfRANS surface pressure on an irregular mesh, and LDC-3D
streamlines from sparse sensors.  Error maps concentrate in
high-gradient regions near the body and wake, consistent with the
theory diagnostic and ablations.  The PhySense-Car visualization is
already shown in Figure~\ref{fig:physense_car} above.  The learned
primitive scaffold is deferred to Appendix~\ref{app:qualitative}.

\subsection{Theory Validation, Ablations, and Cost}
\begin{figure}[!hpbt]
\centering
\includegraphics[width=\linewidth]{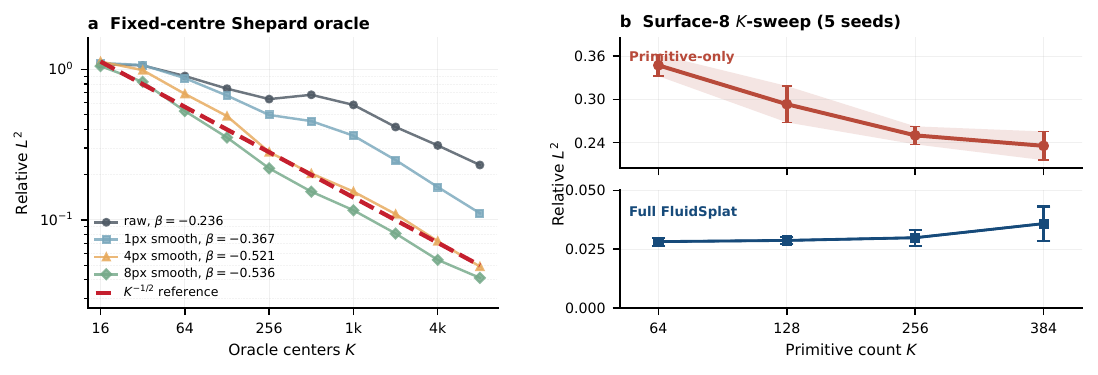}
\caption{Theory validation on the Senseiver cylinder
data.  Left: fixed-center Gaussian Shepard oracle on raw and smoothed
vorticity frames; controlled smoothing moves the empirical exponent
toward the $K^{-1/2}$ first-order reference.  Right: Surface-8 trained
K-sweep (five seeds).  Primitive-only error
improves with $K$ (bias reduction), while the full model is best at
$K\!=\!64$ and degrades at larger $K$, consistent with the
bias--variance trade-off of Corollary~\ref{thm:opt_K}.  The main table
uses $K=384$; the sweep is a single-axis diagnostic.  Exact numbers in
Appendix~\ref{app:additional_tables}.}
\label{fig:nmi_theory_validation}
\end{figure}
\paragraph{Theory validation.}
We separate theory-facing diagnostics from the benchmark tables.  The
oracle experiment in Figure~\ref{fig:nmi_theory_validation} (left)
fixes quasi-uniform centers on the Senseiver cylinder data, reads
amplitudes directly from the full field, and removes optimization
noise.  The raw vorticity field converges slowly due to high-frequency
structure; controlled smoothing moves the empirical exponent toward the
$K^{-1/2}$ first-order reference predicted by
Proposition~\ref{prop:norm_rate} (details in Appendix~\ref{app:additional_tables}).

\begin{table}[t]
\centering
\caption{Controlled AirfRANS Full-8 ablation.  Each row changes one
component of the selected \FS{} configuration and reports test
clean reconstruction error over seeds 1--3 under the same model family
and controlled ablation protocol (the main table reports the final
5-seed benchmark configuration).  Lower is better.}
\label{tab:ablation}
\small
\begin{tabular}{lcc}
\toprule
Variant & Parameters & clean\_recon\_full \\
\midrule
Full \FS{} & 0.236M & \textbf{0.2367 $\pm$ 0.0451} \\
w/o global token & 0.212M & 0.2759 $\pm$ 0.0358 \\
w/o state features & 0.235M & 0.2992 $\pm$ 0.0171 \\
w/o query-to-state attention & 0.210M & 0.3446 $\pm$ 0.0143 \\
residual-only (no primitive output) & 0.209M & 0.3530 $\pm$ 0.0130 \\
primitive-only (no residual decoder) & 0.114M & 0.6524 $\pm$ 0.0158 \\
w/o observation consistency & 0.236M & 0.2620 $\pm$ 0.0121 \\
\bottomrule
\end{tabular}
\end{table}

\begin{table}[t]
\centering
\caption{AirfRANS sensor-count curve (5 seeds).  The 8-sensor
position is tuned; 4/12/16 use analogous arclength schedules.}
\label{tab:airfrans-sensors}
\small
\begin{tabular}{lcccc}
\toprule
\# Sensors & 4 & 8 & 12 & 16 \\
\midrule
\FS{} (ours) & \textbf{0.2997 $\pm$ 0.0328}
             & \textbf{0.2518 $\pm$ 0.0389}
             & \textbf{0.2833 $\pm$ 0.0369}
             & \textbf{0.2813 $\pm$ 0.0521} \\
DeepONet     & 0.3108 $\pm$ 0.0215
             & 0.3289 $\pm$ 0.0330
             & 0.3320 $\pm$ 0.0121
             & 0.3171 $\pm$ 0.0312 \\
\bottomrule
\end{tabular}
\end{table}

\begin{table}[!hbpt]
\centering
\caption{Model size and inference cost on AirfRANS Full-8.  Forward
time is end-to-end (sensor encoding, primitive prediction, and query
decoding) with batch size 1 and 16{,}384 query points on the same GPU.  This table is
reported on AirfRANS rather than cylinder because AirfRANS exercises
the full multi-channel, geometry-conditioned architecture.}
\label{tab:cost}
\small
\begin{tabular}{lcc}
\toprule
Method & Parameters & Forward time (ms) \\
\midrule
\FS{} (ours) & 0.236M & 2.57 $\pm$ 0.20 \\
DeepONet & 0.325M & 2.02 $\pm$ 0.28 \\
RecFNO (reimpl.) & 17.136M & 3.63 $\pm$ 0.97 \\
FLRONet (reimpl.) & 1.332M & 2.59 $\pm$ 0.17 \\
Senseiver official-adapter & 0.945M & 6.12 $\pm$ 0.75 \\
\bottomrule
\end{tabular}
\end{table}

\paragraph{K-sweep.}
Figure~\ref{fig:nmi_theory_validation} (right) shows the trained
model on Surface-8 as $K$ varies.  Primitive-only error improves with
$K$ (bias reduction), while the full model is best at $K\!=\!64$
($0.0282$) and degrades at larger $K$ ($0.0358$ at $K\!=\!384$).
This pattern is consistent with the bias--variance trade-off of
Corollary~\ref{thm:opt_K}: the idealized bound predicts a small optimal
primitive count when $N\!=\!8$.  This supports the view of Gaussian
primitives as a structured scaffold rather than a standalone solution.  The sweep holds all non-$K$
hyperparameters fixed at their jointly tuned $K\!=\!384$ values; it is
a single-axis diagnostic, not a full configuration search.  The main
cylinder table retains the $K\!=\!384$ configuration.

\paragraph{Ablation studies.}
We use AirfRANS Full-8 as the primary ablation dataset because it is
the hardest and most realistic setting in the paper: unstructured
airfoil meshes, geometry-conditioned inputs, three physical output
channels, and surface-only pressure observations.  Cylinder Surface-8
is kept as an auxiliary sanity check in the appendix.  The controlled
AirfRANS ablations remove one design component at a time: (i) the
Gaussian primitive contribution, (ii) the residual decoder, (iii)
query-to-state cross-attention, (iv) global-token/state-feature
conditioning, and (v) observation-consistency loss.  The intended
interpretation is not merely that \FS{} has more parameters, but that
explicit Gaussian state plus an attention-conditioned residual decoder
improves sparse-sensor reconstruction under matched training data and
evaluation metrics.  In particular, the Gaussian primitive field is
not presented as a standalone universal solution; it is a structured,
interpretable spatial scaffold whose best performance comes from being
paired with the residual decoder. The ablation uses three seeds as a
controlled diagnostic.

\paragraph{Sensor-count scaling.}
Table~\ref{tab:airfrans-sensors} shows the AirfRANS sensor-count curve.  \FS{}
improves sharply from 4 to 8 sensors and remains consistently better
than DeepONet at 12 and 16.  The non-monotone trend suggests that
placement and model capacity matter in addition to sensor count.

\section{Conclusion}
\label{sec:discussion}

We introduced \FS{}, pairing a Gaussian primitive scaffold with a
state-conditioned residual decoder for sparse-sensor flow
reconstruction.  The accompanying theory identifies a first-order saturation of the
normalized scaffold and a variance bottleneck on primitive capacity, jointly justifying the architecture design.  \FS{} achieves the best mean error on all four benchmarks evaluated, including the released PhySense-Car sensor-only protocol as an additional 3D surface-field test.

\textbf{Limitation and Discussion.} The LDC-3D benchmark confirms 3D applicability on steady volumetric fields, and PhySense-Car provides an additional 3D surface-pressure check under the released evaluation pipeline; both settings remain steady-state, and extending to unsteady 3D sequences (e.g.\ turbulent snapshots) remains future work.  The theory is dimension-parametric (Appendix~\ref{app:dimension}) and places no inherent 2D restriction.  The optimal least square analysis fixes primitive centers and uses an idealized observation model, serving as a mechanistic guide rather than a finite-sample guarantee.  Because sparse-sensor reconstructions may enter safety-critical loops, uncertainty calibration remains an important companion.

\bibliographystyle{plainnat}
\bibliography{references}

\newpage
\appendix

\section{Notation and Preliminaries}
\label{app:notation}

\paragraph{Asymptotic notation}
For symmetric positive semidefinite matrices $A$ and $B$, we write $A \asymp B$ if there exist constants $c_1,c_2>0$ independent of the main parameters, such that 
\[
c_1 B \preceq A \preceq c_2 B.
\]
where the inequalities are understood in the positive semidefinite sense. For nonnegative scalars, we write $a\lesssim b$ if $a\le Cb$ for a constant $C$ independent of the main parameters, and $a\asymp b$ if both $a\lesssim b$ and $b\lesssim a$ hold.

\paragraph{Gaussian primitive classes.}
Let $\varphi_k$ denote a Gaussian basis function of the form in~\eqref{eq:basis}. Denote $\bm{g}_{K}(\bx)=\sum_{k=1}^{K}\bm{c}_{k}\varphi_{k}(\bx)$ as  a $K$-term unnormalized Gaussian primitive as in~\eqref{eq:basis}. Then denote this class as $\mathcal{G}_{K}=\{g_K;c_k,\mu_k,\sigma_k\}$. Also, we denote the normalized Gaussian primitive class as $\mathcal{F}_{K}=\{f_K=\sum_{k=1}^K\psi_k(x)\bm{a}_k\}$ as in~\ref{eq:model}.

\paragraph{Quasi-uniformity.}
Let $X_K=\{\bmu_k\}_{k=1}^K\subset\Omega$ be a set of centers in a bounded Lipschitz domain $\Omega\subset\mathbb R^d$.  Its fill distance and separation radius are
\[
h_K
:=
\sup_{\bx\in\Omega}\min_{1\le k\le K}\|\bx-\bmu_k\|,
\qquad
q_K
:=
\frac12\min_{i\ne j}\|\bmu_i-\bmu_j\|.
\]
A family $\{X_K\}_{K\ge1}$ is quasi-uniform if there exists a constant $\rho<\infty$, independent of $K$, such that $h_K\le \rho q_K$ .
On a fixed bounded Lipschitz domain, quasi-uniformity implies $h_K\asymp K^{-1/d}$, with constants depending only on $\Omega$, $d$, and $\rho$ \cite{wendland2005}.

\section{Approximation Theory}
\label{app:approx}

\subsection{Universal Approximation}

\begin{proposition}[Universal approximation]
\label{prop:universal}
The unions $\bigcup_{K\ge1}\mathcal G_K$ and
$\bigcup_{K\ge1}\mathcal F_K$ are dense in
$L^2(\Omega;\mathbb R^C)$.
\end{proposition}

\begin{proof}
For the unnormalized class, density follows from the universal
approximation theorem for Gaussian RBF networks~\cite{park1991}: finite
linear combinations of Gaussian radial basis functions are uniformly dense
in $C(\bar\Omega)$. Applying the scalar result componentwise gives density
in $C(\bar\Omega;\mathbb R^C)$. Since $C(\bar\Omega;\mathbb R^C)$ is dense
in $L^2(\Omega;\mathbb R^C)$, the union
$\bigcup_{K\ge1}\mathcal G_K$ is dense in $L^2(\Omega;\mathbb R^C)$.

For the normalized class, first take
$\bm f\in C(\bar\Omega;\mathbb R^C)$. Choose quasi-uniform centers
$X_K=\{\bmu_k\}_{k=1}^K$ with fill distance $h_K\to0$, set Gaussian
scales $\sigma_K\asymp h_K$, take $w_k=1$, and set
$\bm a_k=\bm f(\bmu_k)$. The weights are nonnegative and satisfy
$\sum_{k=1}^K\psi_k(\bx)=1$. Hence
\[
\bm f_K(\bx)-\bm f(\bx)
=
\sum_{k=1}^K
\psi_k(\bx)\bigl(\bm f(\bmu_k)-\bm f(\bx)\bigr).
\]
By the classical Shepard approximation argument~\cite{shepard1968},
$\bm f_K\to\bm f$ uniformly on $\bar\Omega$: the contribution from
nearby centers is controlled by the modulus of continuity of $\bm f$,
while the contribution from distant centers vanishes uniformly by
Gaussian decay and quasi-uniform packing. Therefore
$\bigcup_{K\ge1}\mathcal F_K$ is dense in
$C(\bar\Omega;\mathbb R^C)$, and hence dense in
$L^2(\Omega;\mathbb R^C)$.
\end{proof}

\subsection{Proof of Theorem~\ref{thm:rate}}
\label{app:proof_rate}

We prove~\eqref{eq:rate} for scalar $f$ using isotropic Gaussian networks; the vector case follows componentwise.

\begin{proof}
By the extension theorem for bounded Lipschitz domains, there exists a bounded extension operator $E$ such that $Ef|_{\Omega}=f$ and
\[
\|Ef\|_{H^{s+\varepsilon}(\mathbb{R}^d)}
\le
C_{\Omega}
\|f\|_{H^{s+\varepsilon}(\Omega)}.
\]
Using the standard identification $H^{s+\varepsilon}(\mathbb{R}^d)=F^{s+\varepsilon}_{2,2}(\mathbb{R}^d)$
and the Triebel--Lizorkin embedding
\[
F^{s+\varepsilon}_{2,2}(\mathbb{R}^d) \hookrightarrow F^s_{\tau,q}(\mathbb{R}^d), \qquad \frac{1}{\tau}=\frac{1}{2}+\frac{s}{d}, \qquad \frac{1}{q}=1+\frac{s}{d}
\]
we obtain
\[
|Ef|_{F^s_{\tau,q}(\mathbb{R}^d)}
\le
C
\|Ef\|_{H^{s+\varepsilon}(\mathbb{R}^d)}, 
\]

Applying Theorem~9 of~\cite{hangelbroek2010nonlinear} with
$p=2$ to $Ef$, there exists a $K$-term Gaussian network $g_K$ with adaptive centers and adaptive tension parameters such that
\[
\|Ef-g_K\|_{L^2(\mathbb{R}^d)}
\le
C K^{-s/d}
|Ef|_{F^s_{\tau,q}(\mathbb{R}^d)}.
\]

Combining the previous estimates gives
\[
\|Ef-g_K\|_{L^2(\mathbb{R}^d)}
\le
C K^{-s/d}
\|f\|_{H^{s+\varepsilon}(\Omega)}.
\]

Restricting the estimate to $\Omega$ yields
\[
\|f-g_K\|_{L^2(\Omega)}
\le
\|Ef-g_K\|_{L^2(\mathbb{R}^d)}
\le
C K^{-s/d}
\|f\|_{H^{s+\varepsilon}(\Omega)}.
\]

Finally, isotropic Gaussians are a special case of the Gaussian
primitive in~\eqref{eq:basis} by choosing
$\Sigma_k=\sigma_k^2 I$. Hence the same approximant belongs to
$\mathcal G_K$.

\end{proof}

\subsection{Proof of Proposition~\ref{prop:norm_rate}}
\label{app:proof_norm}

We first prove a localization estimate for the normalized Gaussian Shepard weights.  It states that, although each Gaussian has global support, the normalized weights concentrate their moments at scale $h_K$.

\begin{lemma}[Moment localization]
\label{lem:gaussian_moment}
Let $\Omega\subset\mathbb{R}^d$ be a bounded Lipschitz domain, and let $X_K=\{\bmu_k\}_{k=1}^K\subset\Omega$ be quasi-uniform. Assume further that the Gaussian amplitudes and covariances satisfy $0<w_-\le w_k\le w_+<\infty$,
and $ c_\Sigma h_K^2 I \preceq \Sigma_k \preceq C_\Sigma h_K^2 I, k=1,\dots,K$ with constants  $c_\Sigma, C_\Sigma$ independent of $K$. Then, for every $m\ge 0$,
\[
\sup_{\bx\in\Omega}
\sum_{k=1}^K
\psi_k(\bx)|\bx-\bmu_k|^m
\le
C_m h_K^m,
\]
where $C_m$ is independent of $K$.
\end{lemma}

\begin{proof}
\label{app:gaussian_moment}
Fix $\bx\in\Omega$. By the definition of $h_K$, there exists $k_{\bx}$ such that $|\bx - {\boldsymbol{\mu}_{k_{\bx}}} | \le h_K$.

Since $\Sigma_k\preceq C_\Sigma h_K^2I$, $\varphi_k(\bx) \le \exp\!\left( -\frac{|\bx-\bmu_k|^2}{2C_\Sigma h_K^2} \right),$ as a result
\[
\sum_{k=1}^K
w_k\varphi_k(\bx)|\bx-\bmu_k|^m
\le
w_+
\sum_{k=1}^K
\exp\!\left(
-\frac{|\bx-\bmu_k|^2}{2C_\Sigma h_K^2}
\right)
|\bx-\bmu_k|^m .
\]

Decompose the centers into annuli
\[
A_n(\bx)
:=
\left\{
\bmu_k:
nh_K\le |\bx-\bmu_k|<(n+1)h_K
\right\},
\qquad n=0,1,2,\dots .
\]
By quasi-uniformity, $q_K\gtrsim h_K$, and the number of centers in each annulus satisfies
\[
\#A_n(\bx)\le C(n+1)^d .
\]
Therefore,
\[
\sum_{k=1}^K
w_k\varphi_k(\bx)|\bx-\bmu_k|^m
\le
C h_K^m
\sum_{n=0}^{\infty}
(n+1)^{m+d}
\exp(-c n^2).
\]

Since the Gaussian $\exp(-cn^2)$ dominates any polynomial growth, the series is finite. Then
\[
\sum_{k=1}^K
w_k\varphi_k(\bx)|\bx-\bmu_k|^m
\le
C_m h_K^m .
\]

On the other hand, the lower covariance bound gives $\Sigma_{k_{\bx}} \succeq c_\Sigma h_K^2 I$, we have $ \varphi_{k_{\bx}}(\bx) \ge \exp\!(-1/2c_\Sigma)$. Therefore,
$ \sum_{j=1}^K w_j\varphi_j(\bx) \ge w_-\exp\!\left(-\frac{1}{2c_\Sigma}\right)$. Dividing by the uniform lower bound on the denominator gives
\[
\sum_{k=1}^K
\psi_k(\bx)|\bx-\bmu_k|^m
\le
C_m h_K^m .
\]
\end{proof}

We now apply the moment estimate to the scaffold $f_{\mathrm{prim}}(x)=\sum_k\psi_k(x)f(\mu_k)$ and prove Proposition~\ref{prop:norm_rate}.

\begin{proof}
Since $s>1+d/2$, Sobolev embedding gives
\[
\|\nabla f\|_{L^\infty(\Omega)}
\le
C\|f\|_{H^s(\Omega)}.
\]
In particular, the point values $f(\bmu_k)$ are well-defined. Using the partition-of-unity property $\sum_{k=1}^K\psi_k(\bx)=1$, we have
\[
f(\bx)-f_{\mathrm{prim}}(\bx)
=
\sum_{k=1}^K
\psi_k(\bx)\bigl(f(\bx)-f(\bmu_k)\bigr).
\]
By the Lipschitz bound,
\[
|f(\bx)-f(\bmu_k)|
\le
\|\nabla f\|_{L^\infty(\Omega)}
|\bx-\bmu_k|.
\]
Therefore,
\[
|f(\bx)-f_{\mathrm{prim}}(\bx)|
\le
\|\nabla f\|_{L^\infty(\Omega)}
\sum_{k=1}^K
\psi_k(\bx)|\bx-\bmu_k|.
\]
Applying Lemma~\ref{lem:gaussian_moment} with $m=1$ gives
\[
|f(\bx)-f_{\mathrm{prim}}(\bx)|
\le
C h_K\|\nabla f\|_{L^\infty(\Omega)}.
\]
Taking the $L^2(\Omega)$ norm yields
\[
\|f-f_{\mathrm{prim}}\|_{L^2(\Omega)}
\le
C h_K\|f\|_{H^s(\Omega)}.
\]
Finally, quasi-uniformity on a fixed bounded domain gives
$h_K\asymp K^{-1/d}$. Hence
\[
\|f-f_{\mathrm{prim}}\|_{L^2(\Omega)}
\le
C K^{-1/d}\|f\|_{H^s(\Omega)}.
\]
\end{proof}

\subsection{Residual Decoder}

\begin{proposition}[Residual expressivity]
\label{prop:residual}
Let $\mathcal H$ be a residual decoder class dense in $L^2(\Omega;\mathbb R^C)$.  Then, for any fixed $K\ge 1$, the sum class $
\mathcal F_K+\mathcal H := \{\bm u+\bm h:\bm u\in\mathcal F_K,\ \bm h\in\mathcal H\}$ is dense in $L^2(\Omega;\mathbb R^C)$. Moreover, for any
$\bm f\in L^2(\Omega;\mathbb R^C)$,
\[
\inf_{\bm g\in\mathcal F_K+\mathcal H}
\|\bm g-\bm f\|_{L^2}
\le
\min\left\{
\inf_{\bm u\in\mathcal F_K}\|\bm u-\bm f\|_{L^2},
\inf_{\bm h\in\mathcal H}\|\bm h-\bm f\|_{L^2}
\right\}.
\]
\end{proposition}

\begin{proof}
Since zero amplitudes are allowed, $\bm 0\in\mathcal F_K$.
Since the residual decoder can represent the zero function,
$\bm 0\in\mathcal H$. Hence
\[
\mathcal H\subset \mathcal F_K+\mathcal H,
\qquad
\mathcal F_K\subset \mathcal F_K+\mathcal H .
\]
Because $\mathcal H$ is dense in $L^2(\Omega;\mathbb R^C)$,
the sum class $\mathcal F_K+\mathcal H$ is also dense. The stated
inequality follows directly from the same two inclusions.
\end{proof}

This proposition is not a rate statement. It only records that adding the
residual decoder cannot reduce the approximation class. The purpose of the
Gaussian scaffold is therefore not universal approximation by itself, but to
provide a compact and spatially explicit low-dimensional structure, while the
residual decoder represents the remaining correction.

\section{Estimation Theory}
\label{app:estim}

We treat the bias--variance analysis per scalar component; summing
over $C$ channels introduces only a constant factor.

\subsection{Bias--variance decomposition}

\paragraph{Proof of Theorem~\ref{thm:bv}} \begin{proof}
Let $f_{K}^{*}\in V_{K}$ be the $L^{2}$-projection of $f$.
Since $(f-f_{K}^{*})\perp_{L^{2}}V_{K}$ and
$(f_{K}^{*}-\hat{f}_{K})\in V_{K}$, the Pythagorean
theorem gives
\begin{equation}\label{eq:bias-var-decomp}
\norm{f-\hat{f}_{K}}_{L^{2}}^{2}
=\norm{f-f_{K}^{*}}_{L^{2}}^{2}
+\norm{f_{K}^{*}-\hat{f}_{K}}_{L^{2}}^{2}.
\end{equation}

\paragraph{Bias.}
Since $\bm{f}_{K}^{*}$ is the $L^{2}$-best approximation in $V_{K}$,
it is bounded by the constructive approximant according to assumption (i) $ f_{K}^{*}= \arg\inf_{g\in V_K} \|f-g\|_{L^2}$. This gives the bias bound $\norm{f-f_{K}^{*}}_{L^{2}}^{2} \lesssim K^{-2s/d}\norm{\bm{f}}_{H^{s+\varepsilon}}^{2}$. 

\paragraph{Variance.}
Let $r_X:=(r_1,\ldots,r_N)^\top$ with $r_i=f(x_i)-f_K^*(x_i)$. At the observation points, $y=A c^*+r_X+\varepsilon$. Write $P=(A^\top A)^{-1}A^\top$. Then least square coefficient is
\[
\hat c=P y=
P(Ac^*+r_X+\varepsilon)=
c^*+P(r_X+\varepsilon),
\]
so the coefficient error is
\[
e:=\hat c-c^*=
P(r_X+\varepsilon).
\]
The corresponding function error inside \(V_K\) satisfies
\begin{equation}\label{eq:var-err}
\|f_K^*-\hat f_K\|_{L^2}^2
=\|\sum_k e_k \varphi_k(x)\|_{L^2}^2
= \sum_{j,l} e_j e_l \int_{\Omega} \varphi_j(x) \varphi_l(x) \, dx
=
e^\top G e .
\end{equation}

Taking expectation over the noise, and using
\(\mathbb E[\varepsilon]=0\), we get

\begin{equation}\label{eq:var-decomp}    
\mathbb E[e^\top G e]
=\mathbb E[(r_X+\varepsilon)^TP^T G P(r_X+\varepsilon)]
=r_X^\top P^\top G P r_X
+
\mathbb E[\varepsilon^\top P^\top G P\varepsilon].
\end{equation}

For the first term, residual term, by the lower spectral stability in assumption (ii) $G \preceq \frac{|\Omega|}{c_sN}A^\top A$. Therefore

\[
(Pr_X)^\top G(Pr_X)
\le
\frac{|\Omega|}{c_sN}
(Pr_X)^\top A^\top A(Pr_X)=
\frac{|\Omega|}{c_sN}
\|APr_X\|_2^2 \le \frac{|\Omega|}{c_sN}
\|r_X\|_2^2.
\]
where the last inequality uses that $AP$ is the Euclidean orthogonal projection onto $\mathrm{col}(A)$.

By assumption (iii) $\|r_X\|_2/\sqrt{N}\lesssim \|r\|_{H^s+\varepsilon}$ we obtain.
\begin{equation}\label{eq:var-residual-term}
r_X^\top P^\top G P r_X
\le
\frac{C_r}{c_s}\|r\|_{L^2}^2.
\end{equation}
Thus, the residual term is controlled by the same approximation bias by definition.

For the second term, noise term, use
\[
\mathbb E[\varepsilon^\top P^\top G P\varepsilon]=
\operatorname{tr}
\left(
P^\top G P\,\operatorname{Cov}(\varepsilon)
\right).
\]
Since \(\operatorname{Cov}(\varepsilon)\preceq\sigma^2_{\mathrm{noise}} I_N\),
\[
\mathbb E[\varepsilon^\top P^\top G P\varepsilon]
\le
\sigma^2_{\mathrm{noise}} \operatorname{tr}(P^\top G P)
=\sigma^2_{\mathrm{noise}} \operatorname{tr}(GPP^\top)
=\sigma^2_{\mathrm{noise}} \operatorname{tr}\left(G(A^\top A)^{-1}\right)
\]

The last equality holds for $P=(A^\top A)^{-1}A^\top$. Again by assumption (ii) $A^\top A \succeq c_s\frac{N}{|\Omega|}G$, we obatain
\[
G^{1/2}(A^\top A)^{-1}G^{1/2}
\preceq
\frac{|\Omega|}{c_sN}I_K.
\]

Thus
\[
\operatorname{tr}\left(G(A^\top A)^{-1}\right)=
\operatorname{tr}\left(
G^{1/2}(A^\top A)^{-1}G^{1/2}
\right)
\le
\frac{|\Omega|}{c_sN}K.
\]

As a results,

\begin{equation}\label{eq:var-noise-term}
\mathbb E[\varepsilon^\top P^\top G P\varepsilon]
\le
\frac{|\Omega|}{c_s}\sigma^2_{\mathrm{noise}} \frac KN .
\end{equation}

Combining \cref{eq:bias-var-decomp,eq:var-err,eq:var-decomp,eq:var-residual-term,eq:var-noise-term},

\[
\mathbb E\|f-\hat f_K\|_{L^2}^2
\le
\left(1+\frac{C_r}{c_s}\right)
\|f-f_K^*\|_{L^2}^2
+
\frac{|\Omega|}{c_s}\sigma^2_{\mathrm{noise}} \frac KN .
\]

This proves the theorem.
\end{proof}

\subsection{Simplified Bound and Optimal K}

\paragraph{Proof of Corollary~\ref{thm:opt_K}.}

\begin{proof}    
Treating $K$ as continuous (integer constraint contributes only a constant factor), the objective $F(K)=C_{\mathrm{bias}}^{2}K^{-2s/d}\norm{\bm{f}}_{H^{s+\varepsilon}}^{2}+C_{\mathrm{var}}\sigma^{2}_{\mathrm{noise}}K/N$
has a unique minimiser.  Setting $F'(K)=0$ and solving gives
$K^{(2s+d)/d}\propto(N/\sigma^{2})\norm{\bm{f}}_{H^{s+\varepsilon}}^{2}$;
substituting back gives
$\mathrm{bias}^{2}\sim\mathrm{variance}\sim(\sigma^{2}/N)^{2s/(2s+d)}\norm{\bm{f}}_{H^{s+\varepsilon}}^{2d/(2s+d)}$.
\end{proof}

\section{Dimension-Specific Analysis}
\label{app:dimension}

All theoretical results are parametric in $d$; the two cases of
practical interest specialise as follows.

\begin{center}
\renewcommand{\arraystretch}{1.3}
\small
\begin{tabular}{lcc}
\toprule
Quantity & $d=2$ & $d=3$ \\
\midrule
Fill distance & $h\sim K^{-1/2}$ & $h\sim K^{-1/3}$ \\
Unnorm.\ rate & $K^{-s/2}$ & $K^{-s/3}$ \\
Norm.\ rate & $K^{-1/2}$ & $K^{-1/3}$ \\
Optimal $K^{*}$ & $(N/\sigma^{2})^{1/(s+1)}$ & $(N/\sigma^{2})^{3/(2s+3)}$ \\
Optimized risk scale  & $(\sigma^{2}/N)^{s/(s+1)}$ & $(\sigma^{2}/N)^{2s/(2s+3)}$ \\
\bottomrule
\end{tabular}
\end{center}

\paragraph{Concrete values for $d=2$ (constants set to 1).}
For $N=4/8/16/32$, 
$s=1$: $K^{*}\approx 2/3/4/6$;
$s=2$: $K^{*}\approx 2/2/3/3$;
$s=3$: $K^{*}\approx 1/2/2/2$.

\paragraph{Concrete values for $d=3$ (constants set to 1).}
For $N=4/8/32/128$, 
$s=1$: $K^{*}\approx 2/3/8/18$;
$s=2$: $K^{*}\approx 2/2/4/8$;
$s=3$: $K^{*}\approx 2/2/3/5$.

These values are not direct hyperparameter choices; they indicate the
effective capacity scale implied by the bias--variance trade-off. In the
small-$N$ regime, only a few primitive degrees of freedom can be reliably
supported by the observations. This variance bottleneck motivates using the
primitives as a compact spatial scaffold, with the residual decoder carrying
the remaining fine-scale correction.

\section{Connection to Navier--Stokes Regularity}
\label{app:ns}

\begin{corollary}[Flow field approximation]
\label{cor:ns}
Consider the steady incompressible Navier--Stokes equations with small Reynolds number on
$\Omega\subset\R^{d}$ ($d=2$ or $3$) with $C^{m}$ boundary
($m\ge 2$).  Assume a classical solution
$\bm{u}\in H^{m+\varepsilon}(\Omega;\R^{d})$, $p\in H^{m-1+\varepsilon}(\Omega)$. Then
$\norm{\bm{u}-\bm{u}_{K}}_{L^{2}}=O(K^{-m/d})$ and
$\norm{p-p_{K}}_{L^{2}}=O(K^{-(m-1)/d})$ (for $m\ge 3$), with
estimation-optimal counts
$K^{*}_{\bm{u}}\sim(N/\sigma^{2})^{d/(2m+d)}$ and
$K^{*}_{p}\sim(N/\sigma^{2})^{d/(2(m-1)+d)}$.
\end{corollary}

Apply Theorem~\ref{thm:rate} with $s=m$ (velocity) and $s=m-1$
(pressure); substitute into Theorem~\ref{thm:opt_K}. For smooth domains, smooth data, and regimes where classical steady
solutions exist, such Sobolev regularity assumptions are standard in the
Navier--Stokes regularity theory; see Temam~\cite{temam2024} and
Girault--Raviart~\cite{girault1986}.

\paragraph{Corner singularities.}
At a geometric corner with interior angle $\theta$, regularity is
limited to $s<1+\pi/\theta$ (Grisvard~\cite{grisvard1985};
Kozlov--Maz'ya--Rossmann~\cite{kozlov2000}): $s<1.5$ at a sharp
trailing edge ($\theta\approx 2\pi$), $s<2$ at a blunt one
($\theta\approx\pi$).  This caps the achievable rate on airfoils.

\paragraph{Boundary layers.}
Gradients scale as $\norm{\nabla\bm{u}}\sim Re^{1/2}$ in layers of
thickness $\delta\sim Re^{-1/2}$.  Adaptive learned centers can
concentrate near walls, paralleling mesh refinement; covering the
layer volume $\sim\delta\cdot L^{d-1}$ with diameter-$\delta$
Gaussians suggests $O(Re^{(d-1)/2})$ primitives for the layer alone,
though the exact scaling is geometry-dependent.

\section{Experimental Setup}
\label{app:exp_setup}

\subsection{AirfRANS}
\label{app:datasets}

We use the AirfRANS dataset~\cite{airfrans2022} of 2D RANS
simulations of NACA-series airfoils over ranges of angle of attack
and Reynolds number.  Each sample provides velocity $(u,v)$ and
pressure $p$ on an unstructured mesh plus geometric fields (signed
distance to surface, normal vectors).  The target field is $\R^{C}$
with $C=3$.

\paragraph{Surface-pressure sensor layout.}
The main table uses $N=8$ sensors at normalized arclength positions
$\{0.0625, 0.1875, 0.3125, 0.4375, 0.5625, 0.6875, 0.8125, 0.9375\}$
--- uniform spacing along the airfoil surface.  The sensor-count
ablation uses analogous schedules at $N\in\{4,12,16\}$.

\paragraph{Splits.}
(i)~\emph{Full-8}: all AoA and Reynolds values combined with a
fixed random split.
(ii)~\emph{AoA-8}: fixed Reynolds, disjoint AoA between train and
test (generalization in angle of attack).
(iii)~\emph{Reynolds-8}: fixed AoA, disjoint Reynolds between train
and test.
Split indices are frozen via JSON manifests to ensure
reproducibility.

\paragraph{Query and evaluation.}
During training we sample 2{,}048 query points per case with 60\%
concentration near the airfoil surface; the surface is evaluated
separately with 256 additional query points.  The primary metric is
relative $L^{2}$ error (``clean\_recon\_full'') averaged over the
test set; secondary metrics are surface pressure error, drag/lift
coefficients, and sensor-consistency error.

\subsection{Senseiver Cylinder (Cy\_Taira)}

We use the Cy\_Taira cylinder dataset as distributed with the
Senseiver public repository~\cite{senseiver2023}.  The field is 2D
vorticity behind a cylinder at $Re=100$.  Following the Senseiver
1\%-data protocol, all methods use the same fixed 50 training frames
(split seed 123).  The reported test metric is computed on the 4{,}950
non-training frames.  For methods that require checkpoint selection,
we use a fixed 200-frame validation subset drawn from the non-training
frames; the final reported metric still follows the Senseiver
4{,}950-frame convention.

\paragraph{Interior-8 protocol.}
Sensor positions match the official Senseiver setting (8 interior
sensors placed in the wake).  For this protocol, we cite the
published Senseiver number ($0.039$) as the comparison, and report
our \FS{} result under the same sensor layout, data protocol, and
relative-$L^{2}$ evaluation metric.

\paragraph{Surface-4/8/16 protocol.}
We introduce a physically motivated protocol where sensors are
placed on the cylinder wall, mirroring the AirfRANS surface-sensor
scenario.  The cylinder center (row 55.5, col 13.5) and radius (6\,px)
are determined from the NMI grid geometry ($112\times192$); sensors are
placed at evenly spaced angles and snapped to valid integer grid
locations.  The layouts are frozen across methods:
\begin{center}
\scriptsize
\begin{tabular}{cl}
\toprule
Layout & Grid coordinates [row, col] \\
\midrule
Surface-4 & [55,20], [49,14], [55,7], [62,13] \\
Surface-8 & [55,20], [51,18], [49,14], [51,9], [55,7], [60,9], [62,13], [60,18] \\
Surface-16 & [55,20], [53,19], [51,18], [50,16], [49,14], [50,11], [51,9], [53,8], \\
           & [55,7], [58,8], [60,9], [61,11], [62,13], [61,16], [60,18], [58,19] \\
\bottomrule
\end{tabular}
\end{center}
For this protocol, we run the official Senseiver code with only the
sensor coordinates changed.  Because the original paper does not
publish a surface-sensor setting, we select a strong official-code
configuration and reuse it across seeds with the same fixed 50
training frames; only the model seed is changed.  Surface-4 Senseiver
is evaluated at the fixed 2{,}200-epoch point selected by an explicit
convergence sweep, while Surface-8/16 use the long official-code
schedule; this prevents premature early stopping from dominating the
surface comparison.

\subsection{\FS{} Implementation Details}

\paragraph{Shared core.}
All experiments use the same architecture: per-sensor MLP encoder with
mean/max pooling, a linear primitive head, query-to-primitive
cross-attention with residual connection and feed-forward layer, and a
Fourier-feature residual MLP.  The decoder receives the scaffold
evaluation $\bm{f}_{\mathrm{prim}}(\bx)$, basis mass, global context,
and cross-attention output in all cases.

\paragraph{Cylinder Surface-4/8/16.}
The primitive head predicts an anisotropic,
rotated Gaussian state with $K=256/384/448$ primitives for
Surface-4/8/16 respectively.  The residual path is a Fourier-feature
MLP with 4 hidden layers (width $320/448/512$ for Surface-4/8/16) and
6 frequency bands.  A learned global token is appended.
No per-query geometry features or metadata conditioning are used.

\paragraph{Primitive parameterization (cylinder).}
Centers $\bmu_{k}\in[0,1]^{d}$ (domain normalized).  Scales
$\bsig_{k}$ are bounded for numerical stability: Surface-4/8 use
$[0.01,0.35]$, while Surface-16 uses the sharper range
$[0.005,0.25]$.
Weights $w_{k}\in(0,1)$ via sigmoid.  Amplitudes $\ba_{k}\in\R^{C}$
unconstrained.  Rotation angle $\theta_{k}$ is enabled.

\paragraph{AirfRANS (all splits).}
\FS{} uses the same scaffold-conditioned decoder as in the cylinder
experiments but with axis-aligned Gaussians (rotation disabled) and
lightweight domain conditioning.  The decoder additionally receives
6-D per-query features (signed distance to the airfoil,
nearest-surface normal, normalized arclength, and flow-frame
coordinates) and featurewise metadata modulation from normalized case
descriptors (Reynolds number, angle of attack; modulation gain $0.35$).
The Full and AoA splits use a learned global token; the Reynolds split
disables it.  Primitives: $K=96$; residual MLP: 3 hidden layers, width
192/128/192 for Full/AoA/Reynolds respectively; 6 frequency bands;
scale range $[0.01, 0.35]$.

\paragraph{Training (cylinder).}
Surface-cylinder runs use Adam with weight decay $10^{-6}$, batch size
16, ReduceLROnPlateau (factor $0.5$, patience 25), and gradient
clipping.  Learning rates are $8\times 10^{-4}$ for Surface-4 and
$5\times 10^{-4}$ for Surface-8/16.  Observation-consistency weights
are $\lambda_{\mathrm{obs}}=0.02/0.01/0.005$ for Surface-4/8/16,
respectively.  Minimum training lengths are 350--500 epochs with
early-stopping patience 100--180.  Exact per-run configurations are
archived with the experiment outputs.

\paragraph{Training (AirfRANS).}
Adam with learning rate $10^{-3}$, weight decay $10^{-6}$, batch size
2, ReduceLROnPlateau, gradient clipping at 1.0.
$\lambda_{\mathrm{obs}}=0.05$ and $\lambda_{\mathrm{surface}}=0.2$.  Early stopping patience 3 (in units of
full dataset passes).  Five training seeds per split for the main
table; three seeds for the controlled ablation.

\paragraph{Compute.}
All experiments were run on a single NVIDIA RTX 5060 Ti (16\,GB).
Typical wall-clock training times per run: cylinder Surface
$\approx$20--65\,min; AirfRANS $\approx$50\,min per split-seed.

\section{Baseline Implementations}
\label{app:baselines}

\paragraph{Senseiver (official / official-adapter).}
We run the official Senseiver public code from commit
\texttt{e443eb0} with only sensor coordinates changed for the
surface-cylinder setting.  In AirfRANS we adapt the official
conditioning scheme to the AirfRANS geometry tokens; we label this
variant \emph{Senseiver official-adapter} when it is not a direct
reproduction of a published experiment.

\paragraph{DeepONet.}
A branch-and-trunk architecture where the branch encodes sensor
readings (coordinates + values + optional geometry tokens) and the
trunk encodes normalized query coordinates with Fourier features.
Not an official DeepONet release; we implement the version best
suited to AirfRANS-style sparse sensing.

\paragraph{RecFNO (reimpl.) and FLRONet (reimpl.).}
Operator-style baselines.  RecFNO: pooled sensor tokens $\to$ dense
latent grid $\to$ spectral convolution $\to$ bilinear query
sampling.  FLRONet: radial/softmax sensor-to-grid embedding $\to$ CNN
decoder $\to$ bilinear sampling.  We clearly label these as
reimplementations; neither matches the exact architecture of any
single published paper.
For FlowBench LDC-3D, we use compact configurations (Senseiver
reimpl.\ 0.72M, RecFNO reimpl.\ 0.68M) tuned to the $64^3$ volumetric
task; larger configurations exhibited training instability on this
benchmark, with individual seeds diverging above 0.15 relative $L^2$.
The compact variants train stably across all five seeds and provide a
fairer representation of each method's capacity on this task.

\paragraph{POD-Ridge.}
Classical baseline: interpolate training fields to a regular grid,
take SVD to obtain POD modes, fit ridge regression from sensor
values to POD coefficients, project back to the query mesh.
Deterministic (no seed-level variance).

\section{Additional Results}
\label{app:additional_tables}

\subsection{Cylinder surface results: exact seed values}
\label{app:seed_stability}

Table~\ref{tab:cylinder_seed_values} lists the exact seed-level means
used to compute the main cylinder table.  All entries use the
surface-sensor layouts and the same fixed 50-frame
training split.  Senseiver official uses the public code with only the
surface coordinates and convergence schedule specified for this
surface protocol; Surface-4 is reported at a fixed 2{,}200 epochs
following the convergence sweep, while Surface-8/16 use the long-run
official-code schedule.

\begin{table}[h]
\centering
\caption{Exact per-seed relative $L^{2}$ values for the
surface-cylinder protocol.  These are seed-level means over the 4{,}950
test frames.}
\label{tab:cylinder_seed_values}
\scriptsize
\resizebox{\linewidth}{!}{%
\begin{tabular}{llccccc}
\toprule
Setting & Method & 123 & 124 & 125 & 126 & 127 \\
\midrule
Surface-4  & \FS{} (ours)       & 0.0679 & 0.0640 & 0.0572 & 0.0662 & 0.0654 \\
Surface-4  & Senseiver official & 0.0718 & 0.0643 & 0.0940 & 0.0817 & 0.0675 \\
Surface-4  & DeepONet           & 0.0981 & 0.0925 & 0.0916 & 0.0901 & 0.0900 \\
Surface-4  & FLRONet (reimpl.)  & 0.0702 & 0.0688 & 0.0789 & 0.0683 & 0.0604 \\
Surface-4  & RecFNO (reimpl.)   & 0.0783 & 0.0542 & 0.0975 & 0.0773 & 0.0655 \\
\midrule
Surface-8  & \FS{} (ours)       & 0.0479 & 0.0298 & 0.0348 & 0.0361 & 0.0304 \\
Surface-8  & Senseiver official & 0.0676 & 0.0603 & 0.0495 & 0.0602 & 0.0526 \\
Surface-8  & DeepONet           & 0.1272 & 0.1419 & 0.1086 & 0.1308 & 0.1420 \\
Surface-8  & FLRONet (reimpl.)  & 0.0579 & 0.0592 & 0.0796 & 0.0657 & 0.0766 \\
Surface-8  & RecFNO (reimpl.)   & 0.1138 & 0.0935 & 0.0876 & 0.1080 & 0.1092 \\
\midrule
Surface-16 & \FS{} (ours)       & 0.0166 & 0.0154 & 0.0170 & 0.0156 & 0.0152 \\
Surface-16 & Senseiver official & 0.0571 & 0.0722 & 0.0464 & 0.0575 & 0.0450 \\
Surface-16 & DeepONet           & 0.0959 & 0.0957 & 0.0960 & 0.0950 & 0.0905 \\
Surface-16 & FLRONet (reimpl.)  & 0.0425 & 0.0479 & 0.0515 & 0.0448 & 0.0432 \\
Surface-16 & RecFNO (reimpl.)   & 0.0658 & 0.0742 & 0.0635 & 0.0723 & 0.0768 \\
\bottomrule
\end{tabular}
}
\end{table}

\FS{} has the best five-seed mean in all three surface-sensor regimes.
The strongest non-ours method changes with sensor count: FLRONet is
closest at Surface-4/16, while Senseiver official is closest at
Surface-8 after long convergence-controlled training.

\subsection{Senseiver interior-8: official and local reruns}

For transparency we include both the published Senseiver number and
our local reruns of the official Senseiver code on the interior-8
protocol:

\begin{center}
\small
\begin{tabular}{lc}
\toprule
Setting & clean\_recon\_full \\
\midrule
Senseiver official (published) & 0.039 \\
Senseiver official-adapter (local rerun, 5 seeds) & 0.4118 $\pm$ 0.0135 \\
\FS{} (ours, single run) & \textbf{0.0097} \\
\bottomrule
\end{tabular}
\end{center}

The local adapter run is included only as a diagnostic: for the
Interior-8 leaderboard comparison we cite the published Senseiver
number, which is the most faithful representation of the original
method.  See Section~\ref{app:senseiver_notes} for the surface-protocol
official-code details.

\subsection{Theory-validation exact numbers}
\label{app:nmi_theory_exact}

This subsection reports the exact values underlying
the theory-validation figure in the main paper.  The oracle smoothing
diagnostic uses the official Senseiver Cy\_Taira data, excludes the
fixed 50 training frames, samples 30 evenly spaced held-out frames, and
uses nested farthest-point centers with $\alpha=0.3$.  Smoothing
applies \texttt{scipy.ndimage.gaussian\_filter} with
$\sigma=\text{smooth\_px}$ (Gaussian kernel in grid-pixel units;
$0$=raw, $1$, $4$, $8$ px).  Fitted exponents are log-log linear
regressions over $K\in\{16,32,64,128,256,512,1024,2048,4096,8192\}$.
The roughness column reports the normalized gradient magnitude:
$\text{roughness}=\|\nabla f\|_{\mathrm{rms}}/\|f\|_{\mathrm{rms}}$,
computed over valid grid points.  This controlled smoothing diagnostic
is not used in any benchmark evaluation or trained model; it exists
solely to separate the smooth-field approximation mechanism from the
high-frequency roughness of the unsmoothed benchmark data.
The trained K-sweep uses raw data under the Surface-8 protocol.

\begin{table}[t]
\centering
\caption{Exact oracle Shepard smoothing diagnostics on the Senseiver cylinder data.  Smoothing is applied to the vorticity frames in grid-pixel units before the fixed-center oracle interpolation; no learned model is used.}
\label{tab:nmi_oracle_smoothing_exact}
\scriptsize
\begin{tabular}{lccccc}
\toprule
Setting & roughness & fitted exponent & $K=16$ & $K=128$ & $K=2048$ \\
\midrule
raw & 0.5241 & -0.236 & 1.0851 $\pm$ 0.0866 & 0.7452 $\pm$ 0.0189 & 0.4144 $\pm$ 0.0022 \\
1px & 0.3471 & -0.367 & 1.1039 $\pm$ 0.1031 & 0.6703 $\pm$ 0.0229 & 0.2498 $\pm$ 0.0014 \\
4px & 0.1691 & -0.521 & 1.1276 $\pm$ 0.1137 & 0.4897 $\pm$ 0.0254 & 0.1093 $\pm$ 0.0010 \\
8px & 0.1212 & -0.536 & 1.0529 $\pm$ 0.1005 & 0.3527 $\pm$ 0.0119 & 0.0812 $\pm$ 0.0010 \\
\bottomrule
\end{tabular}
\end{table}

\begin{table}[t]
\centering
\caption{Exact trained Surface-8 K-sweep values.  All rows use the Senseiver cylinder data, the Surface-8 sensor layout, the fixed 50-frame training split, and evaluate relative $L^2$ on 4{,}950 held-out frames.  Values are mean $\pm$ seed standard deviation over seeds $\{123,\ldots,127\}$.}
\label{tab:nmi_trained_k_sweep_exact}
\small
\begin{tabular}{lcc}
\toprule
$K$ & Full \FS{} & Primitive-only \\
\midrule
64 & 0.0282 $\pm$ 0.0017 & 0.3471 $\pm$ 0.0149 \\
128 & 0.0287 $\pm$ 0.0015 & 0.2933 $\pm$ 0.0252 \\
256 & 0.0298 $\pm$ 0.0033 & 0.2501 $\pm$ 0.0125 \\
384 & 0.0358 $\pm$ 0.0073 & 0.2356 $\pm$ 0.0199 \\
\bottomrule
\end{tabular}
\end{table}

\subsection{\FS{} $K$-ablation}

We use $K$-sweeps as a model-design diagnostic rather than as a
headline benchmark: the final surface-cylinder main-table runs tune
$K$ jointly with decoder width, attention dimension, scale range, and
observation-consistency weight and remain fixed in the main results.
The K-sweep varies only $K$ while holding all other hyperparameters at
their $K\!=\!384$-tuned values (residual hidden dim 448, attention dim
160, 4 heads, scale range $[0.01, 0.35]$, $\lambda_{\mathrm{obs}}=0.01$).
In particular, Surface-8 keeps the selected $K=384$ configuration in
the main cylinder table; the smaller-$K$ points demonstrate the
single-axis effect of primitive count under otherwise fixed
architecture.  The diagnostic sweep illustrates that
primitive-only models improve with $K$ yet remain far worse than the
full model at all tested $K$, while the full model degrades at larger
$K$, consistent with the variance bottleneck mechanism.

\begin{center}
\small
\begin{tabular}{lcc}
\toprule
Configuration & Qualitative trend & Use in final model \\
\midrule
Gaussian-only (no residual) & improves with $K$ but remains far worse & ablation only \\
\FS{} (with residual)       & best at small $K$, degrades at large $K$ & jointly tuned with decoder \\
\bottomrule
\end{tabular}
\end{center}

This pattern is consistent with the bias--variance trade-off
prediction: increasing primitive count alone is insufficient under
sparse observations, so the final model relies on the Gaussian state
for structured low-dimensional geometry and the residual decoder for
high-frequency correction.

\section{PhySense-Car: Protocol Details}
\label{app:physense_car}

This appendix supplements the PhySense-Car results reported in main
text Section~\ref{sec:physense_car} (Table~\ref{tab:physense_car},
Figure~\ref{fig:physense_car}) with the exact protocol used.

\paragraph{Data and split.}
We use the official PhySense-Car raw files and the official pressure
normalization constants.  Each case contains $\sim\!95$k surface
points with per-point coordinates and pressure values.  Cases 1--75
are used for training and cases 76--100 for testing, matching the
released split.

\paragraph{Sensor sampling.}
For the random-sensor protocol we use the official fixed seed:
\texttt{torch.manual\_seed(1)} followed by \texttt{torch.randperm}, so
that the $N\in\{15,30,50,100,200\}$ sensors are identical to those
used by the released PhySense evaluation code.  The same sensor set is
used for both the released PhySense checkpoint and \FS{}.

\paragraph{Conditioning.}
To match the released sensor-only PhySense model, all condition inputs
such as speed and yaw are set to zero; the model receives only sensor
coordinates and pressure readings.

\paragraph{Metric.}
We report the normalized relative-$L^2$ error using the same
$3\sigma$ pressure mask and the same normalization constants as the
released PhySense evaluation code, so that the \FS{} numbers in
Table~\ref{tab:physense_car} are directly comparable to the four
baselines quoted from PhySense~\cite{ma2025physense} Table~3.

\paragraph{Training.}
Our \FS{} run is trained for a fixed 301 epochs with random sensor
counts $N\sim\mathrm{Uniform}[10,200]$ per training sample, so that a
single checkpoint generalizes across the five evaluation sensor
counts.  The final checkpoint is evaluated directly without
validation-based checkpoint selection.

The qualitative comparison in main text Figure~\ref{fig:physense_car}
uses two held-out cases at $N=50$ with the same random sensors for
both methods.

\section{Qualitative Reconstructions}
\label{app:qualitative}
The qualitative reconstruction figures in the main paper are generated from archived trained
checkpoints and held-out samples; no additional smoothing or
post-processing is applied to the predictions.  The AirfRANS panel uses
display-normalized centered pressure only for visual contrast:
subtracting the crop median and applying a robust asymmetric display
scale does not change the reconstruction error, and all quantitative tables use the original multi-channel physical field.  
We visualize the learned scaffold on the cylinder domain shown in Figure~\ref{fig:cylinder_scaffold_appendix}, where the regular geometry makes spatial structure directly visible. On AirfRANS, the scaffold's interpretability derives from the same explicit parameterization (inspectable centers, scales, and amplitudes); its quantitative contribution is confirmed via ablation (Table~\ref{tab:ablation}).

\begin{figure}[h]
\centering
\includegraphics[width=0.62\linewidth]{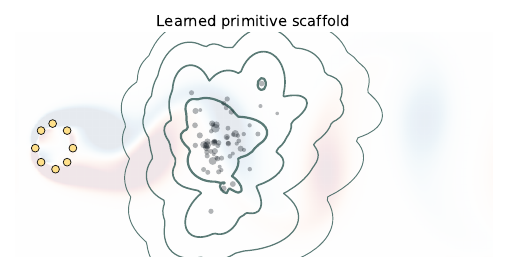}
\caption{Appendix diagnostic for the Senseiver cylinder Surface-8
run.  The panel visualizes primitive influence (maximum weighted
Gaussian basis response) and the highest-weight primitive centers over
a faint ground-truth vorticity context background used only for spatial
reference.  This is not a primitive-only prediction field; it shows how
the structured latent state places support near the body and wake before
the state-conditioned residual decoder forms the final reconstruction.}
\label{fig:cylinder_scaffold_appendix}
\end{figure}

These qualitative panels are intentionally not presented as evidence
that the Gaussian primitive field alone solves the task.  Instead, they
support the central interpretation used throughout the paper:
primitives provide an interpretable, geometry-aware latent scaffold,
whereas the state-conditioned residual decoder supplies the missing
high-frequency correction.  This is why the full model can track the
cylinder vortex street and AirfRANS near-body pressure structure even
when the primitive-only output field is too smooth.

\section{Senseiver Reproduction Notes}
\label{app:senseiver_notes}

\paragraph{What we report vs.\ what Senseiver reports.}
The main-text Cylinder Interior-8 comparison uses the published
Senseiver number ($0.039$) as the baseline.  This avoids treating a
local reproduction attempt as the official method.  Importantly:
\begin{itemize}
\item For the Interior-8 comparison we quote the published number,
avoiding any implied reproduction failure.
\item For the Surface-$K$ protocol (which has no published baseline),
we run the official Senseiver code with the sensor coordinates moved
to the cylinder wall, using the same fixed 50-frame training split as
the other methods.
\end{itemize}

\paragraph{Training configuration used.}
For the surface protocol, seed 123 is the selected tuning-best
official-code Senseiver run for each sensor count.  Seeds 124--127
reuse the same selected configuration and the same fixed training
frames; only the model seed changes.  This gives both a best-of-five
view and a seed-level stability view under the surface-sensor protocol.

\paragraph{Interior-8 reproduction.}
To verify the published Senseiver number we ran the official code on the
Interior-8 layout with the default architecture
(latent dim $128$, $64$ latents, $4$ attention heads, $3$ encoder
layers, $2$ self-attention layers per block) and
the same $70/15/15$ train/val/test split.
Training used the Adam optimizer with an initial learning rate of
$10^{-3}$, ReduceOnPlateau scheduling (factor $0.5$, patience $3$,
minimum $10^{-5}$), and gradient clipping at norm $1.0$.
Over $5$ seeds ($123$--$127$), the test-set relative $L^{2}$ error
at epoch $340$ is $0.0390\pm0.0075$ (mean$\pm$std), matching the
published value.

\paragraph{Clarifications on naming.}
We use the following convention throughout the paper and appendix:
\emph{``Senseiver official''} refers to the published
Senseiver~\cite{senseiver2023} number or to the official code with
only sensor coordinates adjusted for a new protocol;
\emph{``Senseiver official-adapter''} refers to a local
reproduction where architecture/hyperparameters were adapted beyond
the sensor coordinates.  \emph{``RecFNO (reimpl.)''} and
\emph{``FLRONet (reimpl.)''} are our re-implementations, not
official releases.

\end{document}